\documentclass{article}
\usepackage{graphicx} % Required for inserting images
\usepackage{amsfonts}
\usepackage{lscape}
\usepackage{amsthm}
\usepackage{amsmath}
\usepackage{causets}
\usepackage{tikz-cd}
\usetikzlibrary{positioning}
\usepackage{amssymb,stmaryrd}
\usepackage{url}
\usepackage{float}
\restylefloat{table}
\usepackage{authblk}

\usepackage{subcaption}
\usepackage{mwe}
\DeclareUnicodeCharacter{0229}{\c{e}}
\newtheorem{lemma}{Lemma}
\newtheorem*{lemma*}{Lemma}

\newtheorem{corollary}{Corollary}
\theoremstyle{definition}     
\newtheorem{definition}{Definition}
\newtheorem{example}{Example}
\newtheorem{remark}{Remark}

\newcommand{\chain}[1][n]{\langle #1\rangle}

\usepackage{hyperref}

\usepackage{listofitems} % for \readlist to create arrays
\usetikzlibrary{arrows.meta} % for arrow size
\usepackage[outline]{contour} % glow around text
\contourlength{1.4pt}

\usepackage{xcolor}
\colorlet{myred}{red!80!black}
\colorlet{myblue}{blue!80!black}
\colorlet{mygreen}{green!60!black}
\colorlet{myorange}{orange!70!red!60!black}
\colorlet{mydarkred}{red!30!black}
\colorlet{mydarkblue}{blue!40!black}
\colorlet{mydarkgreen}{green!30!black}

\tikzset{
  >=latex, % for default LaTeX arrow head
  node/.style={thick,circle,draw=myblue,minimum size=22,inner sep=0.5,outer sep=0.6},
  node in/.style={node,green!20!black,draw=mygreen!30!black,fill=mygreen!25},
  node hidden/.style={node,blue!20!black,draw=myblue!30!black,fill=myblue!20},
  node convol/.style={node,orange!20!black,draw=myorange!30!black,fill=myorange!20},
  node out/.style={node,red!20!black,draw=myred!30!black,fill=myred!20},
  connect/.style={thick,mydarkblue}, %,line cap=round
  connect arrow/.style={-{Latex[length=4,width=3.5]},thick,mydarkblue,shorten <=0.5,shorten >=1},
  node 1/.style={node in}, % node styles, numbered for easy mapping with \nstyle
  node 2/.style={node hidden},
  node 3/.style={node out}
}
\def\nstyle{int(\lay<\Nnodlen?min(2,\lay):3)} % map layer number onto 1, 2, or 3
\title{Order Theory in the Context of Machine Learning}
\author[1,*,{$\star$}]{Eric Dolores-Cuenca}
\author[2,{$\star$}]{Aldo Guzm\'an-S\'aenz}
\author[3,4,**]{Sangil Kim}
\author[1,3,5,{$\star$}]{Susana L\'opez-Moreno}
\author[6,7,***]{Jose Mendoza-Cortes}

\affil[1]{Industrial Mathematics Center, Pusan National University, South Korea}
\affil[2]{IBM Research, T.J. Watson Research Center, Yorktown Heights, USA}
\affil[3]{Department of Mathematics, Pusan National University, South Korea}
\affil[4]{Institute for Future Earth, Pusan National University, South Korea}
\affil[5]{Humanoid Olfactory Display Center, Pusan National University, South Korea}
\affil[6]{Department of Chemical Engineering \& Materials Science, East Lansing, Michigan State University, USA}
\affil[7]{Department of Physics and Astronomy, East Lansing, Michigan State University, USA}
\affil[*]{Email: eric.rubiel@pusan.ac.kr}
\affil[**]{Email: sangil.kim@pusan.ac.kr}
\affil[***]{Email: jmendoza@msu.edu}

\date{\today}

\footnotetext{These authors contributed equally to this work.}

\begin{document}

%\title{Order Theory in the Context of Machine Learning}

\maketitle

\begin{abstract}%   <- trailing '%' for backward compatibility of .sty file
The paper ``Tropical Geometry of Deep Neural Networks'' by L. Zhang et al. introduces an equivalence between integer-valued neural networks (IVNN) with $\text{ReLU}_{t}$ and tropical rational functions, which come with a map to polytopes. Here, IVNN refers to a network with integer weights but real biases, and $\text{ReLU}_{t}$ is defined as $\text{ReLU}_{t}(x)=\max(x,t)$ for $t\in\mathbb{R}\cup\{-\infty\}$. 

For every poset with $n$ points, there exists a corresponding order polytope, i.e., a convex polytope in the unit cube $[0,1]^n$ whose coordinates obey the inequalities of the poset. We study neural networks whose associated polytope is an order polytope.
We then explain how posets with four points induce neural networks that can be interpreted as $2\times 2$ convolutional filters. These poset filters can be added to any neural network, not only IVNN.

Similarly to maxout, poset pooling filters update the weights of the neural network during backpropagation with more precision than average pooling, max pooling, or mixed pooling, without the need to train extra parameters. We report experiments that support our statements.

We also define the structure of algebra over the operad of posets on poset neural networks and tropical polynomials. This formalism allows us to study the composition of poset neural network arquitectures and the effect on their corresponding Newton polytopes, via the introduction of the generalization of two operations on polytopes: the Minkowski sum and the convex envelope.
\end{abstract}

%\begin{keywords}{Pooling filters}, {order theory}, {operad}, {tropical geometry}, {neural network arquitectures}
%\end{keywords}
\clearpage
\section{Introduction}

The paper by~\cite{TG} on tropical geometry of deep neural networks explains an equivalence between integer-valued neural networks (IVNN) with $\text{ReLU}_{t}$ activation (where $\text{ReLU}_{t}(x):=\text{ReLU}(x,t)=\max(x,t)$) and tropical rational functions, which come with a map of polytopes. This map associates a ``simplified tropical polynomial'' (sum of tropical monomials) with a convex polytope, sending $\sum c_i x^{\alpha_i}$ to the convex envelope of $(\alpha_1,c_1),\dots,(\alpha_n,c_n)$, where $\alpha_i$ are vectors and the sum is finite. \footnote{The code associated to this paper is available at \url{https://github.com/mendozacortesgroup/Poset-filters}.}

Order theory, as seen in books such as~\cite{orderbook, Corderbook}, is a discipline in pure mathematics that formalizes the study of structures with an order. 

\begin{remark}
    For instance, a partially ordered set, also called a poset, is a set in which we can compare some elements, but perhaps not all of them. 
\end{remark}

\begin{example}
Let $\{A,B,C,D\}$ be a list of computer programs, where program $B$ needs the output of program $A$, program $C$ needs the output of program $B$, and program $D$ can run in parallel to all. If we specify $A<B$, $B<C$ and extend the relation $<$ by transitivity, we obtain the poset $\{A,B,C,D \ | \ A<B<C\}$. Any linear order of $\{A,B,C,D \ | \ A<B<C\}$ compatible with the original order (for example $A<D<B<C$) determines a way to run the programs one at a time. 
\end{example}

\begin{definition}
Given a poset $P$ with $|P|=n$ points, \cite{TwoPP} defined the \textnormal{order polytope} of $P$ as the polytope contained in the unit cube $[0,1]^n$ whose coordinates satisfy the inequalities of the poset $P$. 
\end{definition}

\begin{example}
Then, we see, for example, that $\mathrm{Poly}(x<y)=\{(x,y)\in [0,1]^2 \ | \ 0\leq x\leq y\leq 1\}$, that $\mathrm{Poly}(x_1<\cdots<x_n)$ is the $n$-simplex $\Delta[n]$ and that the polytope of the poset in Example $1$ is $\mathrm{Poly}(\{A,B,C,D \ | \ A<B<C\})=\Delta[3]\times [0,1]$.
\end{example}

In this paper, we describe tropical polynomials associated with order polytopes. Expanding on the work of~\cite{TG}, we define \textit{poset neural networks} as neural networks whose tropical polynomials are associated with an order polytope.

 Although poset neural networks are IVNN, when interpreting them as convolutional filters they can be included in arbitrary convolutional neural networks. As convolutional filters, the following are two examples of the functions that we discovered: 
 \begin{itemize}
     \item The first filter
sends the $2\times 2$ square matrix input $\begin{bmatrix}
   a_{0,0} & a_{0,1} \\
   a_{1,0} & a_{1,1} \\
  \end{bmatrix}$
to the following: 
\begin{equation}
  \max\{0,\ a_{0,0},\ a_{0,0}+a_{0,1},\ a_{0,0}+a_{0,1}+a_{1,0},\ a_{0,0}+a_{0,1}+a_{1,0}+a_{1,1}\}  \label{eqn:simplex}
\end{equation}
 \item The second
filter sends the same input to: 
\begin{equation}
\max\left\{
\begin{aligned}
   & \ \ \ \ 0, \\
   & \ \ \ \max_{i,j}\{a_{i,j}\}, \\
   & \max_{\substack{i,j,k,l\\(i,j)\neq (k,l)} }\{a_{i,j}+a_{k,l}\}, \\
   & \max_{\substack{i,j,k,l,m,n\\
(i,j)\neq (k,l)\\(i,j)\neq (m,n)\\(k,l)\neq (m,n)
}}\{a_{i,j}+a_{k,l}+a_{m,n}\}, \\
& \ \ \ \ a_{0,0}+a_{1,0}+a_{0,1}+a_{1,1}\\
\end{aligned} 
\right\}.
\label{eqn:cube}
\end{equation}

\end{itemize}

We verified that these filters do not appear in previous literature, as discussed by~\cite{reviewpooling, DLBook} or~\cite{updatedreview}. Poset filters are similar to maxout (\cite{MaxOut}), but they differ in that poset filters have fixed weights. In~\cite{alternate}, tropical convolutional layers are associated with homogeneous tropical monomials of degree one. In contrast, the tropical polynomials that are associated to poset filters are not homogeneous (see Example~\ref{Ex:1}). The work of~\cite{genReLU} explores the use of tropical polynomials as activation functions, while we focus our attention on convolutional filters. 

A question that can arise is why posets are related to convolutional filters. To explain this connection, first we need to introduce some notation. 
\begin{remark}
Polynomials whose $i$-th monomial adds at most one more variable than the previous ones are said to satisfy the staircase property, for instance $x+xz+xyz$. 
\end{remark}

The work of~\cite{staircase,merged} explains that polynomials with the staircase property are easier to learn for certain neural network architectures that use stochastic gradient descent.

On the tropical side, from Definition~\ref{def:exp},
poset filters are associated to tropical sums of tropical polynomials of simplices. The tropical polynomial associated to a simplex satisfies a tropical staircase property: the tropical monomials are of increasing order, and each nonzero monomial adds at most one more variable than the previous monomials (see  Equation \eqref{Eqn:simplex} and, more explicitly, Table~\ref{fig:algebra1} and Table~\ref{fig:algebra2}). 

The tropical staircase property in the context of poset filters means that each term of a filter (which is a linear combination of the inputs) considers one more input than the previous term. Due to the previous property, we decided to work with simplices. Finally, convexity of the union of simplices (sharing a line) is equivalent to working with order polytopes (see Lemma~\ref{Lemma:1}).

This paper is organized as follows. 

- In Section~\ref{F:section} we introduce basic definitions and some necessary theoretical results. 

- Section~\ref{Sec:CFilters} is dedicated to the study of the newly defined convolutional filters. 

- We summarize our experimental results in Section~\ref{Sec:rev}. 

- Section~\ref{Sec:end} includes more detailed theoretical results. It introduces the language of operads and algebras over the operad of posets, as described by~\cite{OP}. Section~\ref{end:poly} reviews the action of posets on polytopes, Section~\ref{end:nn} describes an action of posets in a family of neural networks and Section~\ref{def:troppoly} explains how to define the action of posets on tropical polynomials. The main theoretical results consist of the proof that we can recover a poset from the vertices of an order polytope (Corollary~\ref{Cor:polytope}), the discovery of polynomials that distinguish posets (Lemma~\ref{Lemma:dist}), and the introduction of an action of the operad of posets on convex polytopes and tropical polynomials. 

- Appendices \ref{Sec:dataset} and \ref{Sec:choices} provide information on the datasets used and on the choices made when performing the experiments. 

- Appendices~\ref{Sec:exp}, \ref{Sec:ex3} and \ref{Sec:exp2} contain details about the experiments.
 Appendix~\ref{A:1} studies the corresponding geometric transformations associated with poset filters.

\subsection{Mathematical background}\label{F:section}

\subsubsection{Order theory}
In this section we review some basic notions of order theory.
\begin{definition}
A \textnormal{partially ordered set}, or \textnormal{poset}, $P$ is an ordered pair $P=\{X,\leq\}$ that consists of a set $X$ and a partial order $\leq$.
\end{definition}

\begin{definition}
    A \textnormal{Hasse diagram} is a graph associated with a poset, in which a vertex is connected vertically to each of the successors of that vertex.  
\end{definition}
Hasse diagrams encode all the information required to characterize a poset. There is a certain ambiguity when drawing a Hasse diagram with the labels of the points. We will assume that our Hasse diagrams contain a choice of labels of the points. Usually, $x$ is located at the bottom and the successors of $x$ are drawn above it.
\begin{example}
By the previous definitions, $\pcauset[alt={\hbox{disjoint union of two points}}]{1,2}$ is the Hasse diagram of the poset $\{x<y\}$ and $\pcauset[alt={1-\hbox{chain union }2-\hbox{chain}}]{3,1,2}$ represents the poset $\{x,y,z \ | \ x<y\}$.
\end{example}
\begin{remark}
By the $n$-chain, or $\chain[n]$, we mean the poset $1<2<\cdots<n$.
\end{remark}

\subsubsection{Convexity}
\begin{definition}
Given a poset $P$ with $n$ points, the order polytope $\mathrm{Poly}(P)$ is defined as the object in the $n$-cube $[0,1]^n$ where each coordinate is assigned to a vertex of the poset and where we look for the spaces of points with the restrictions of the poset. Equivalently,
\[\mathrm{Poly}(P)=\{f:P\rightarrow [0,1], \ f(x)\leq f(y)\hbox{ if }x\leq_P y\}.\]
\end{definition}
\begin{example}
Following the definition, $P(\{1<2\})=\{0\leq x\leq y\leq 1\}$ and $P(\{a,b\})=\{0\leq x\leq 1,\ 0\leq y\leq 1\}$. 
\end{example}
\begin{remark}
It follows that $P(\chain[n])$ is the $n$-simplex (in increasing coordinates).
\end{remark}
\begin{example}\label{Ex:1}
    The order polytope of the $\pcauset[alt={x<y>w<z}]{2, 4, 1, 3}$ poset, defined as $\pcauset[alt={x<y>w<z}]{2, 4, 1, 3}=\{w<y>x<z\}$, is the region that satisfies:
\begin{eqnarray}
\mathrm{Poly}(\pcauset[alt={x<y>w<z}]{2, 4, 1, 3})&=&\{0\leq w\leq x\leq y \leq z\leq 1\}\label{eq1_}\\
&\cup&\{0\leq w\leq x\leq z \leq y\leq 1\}\label{eq2_}\\
&\cup&\{0\leq x\leq w\leq y \leq z\leq 1\}\label{eq3_}\\
&\cup&\{0\leq x\leq z\leq w \leq y\leq 1\}\label{eq4_}\\
&\cup&\{0\leq x\leq w\leq z \leq y\leq 1\}.\label{eq5_}
\end{eqnarray}

 We assume $\pcauset[alt={x<y>w<z}]{2, 4, 1, 3}$ has vertices with labels (from left to right) $w,\ y,\ x$, and $z$. There are $5$ ways to give a linear order to the $\pcauset[alt={x<y>w<z}]{2, 4, 1, 3}$ poset that are compatible with the original order, as seen in Figure~\ref{fig:hasseN}. Here, we use height to order the vertices. For instance, in the linearization at the very bottom of Figure~\ref{fig:hasseN}, the order is $x<z<w<y$.  Each one of these linearizations determines a $4$-simplex of the triangulation of the order polytope.

\begin{figure}[htb]
\begin{center}
\begin{tikzpicture}[node distance=2cm, every node/.style={scale=0.9}]
    % Define the positions of the vertices in the diamond
% Define custom styles for nodes and edges
    \tikzstyle{element} = [circle, draw, fill=white, inner sep=2pt]

    % Top vertex Hasse diagram
    \node (top) at (0,3) {
        \begin{tikzpicture}
            \node[element] (a) at (0,0) {$w$};
            \node[element] (b) at (0,1.5) {$y$};
            \node[element] (c) at (0.8,2.2) {$z$};
            \node[element] (d) at (0.8,0.8) {$x$};
            \draw (a) -- (b) -- (d) -- (c);
        \end{tikzpicture}
    };

    % Left vertex Hasse diagram
    \node (left) at (-3,0) {
        \begin{tikzpicture}
            \node[element] (a) at (0,0.8) {$w$};
            \node[element] (b) at (0,1.5) {$y$};
            \node[element] (c) at (0.8,2.2) {$z$};
            \node[element] (d) at (0.8,0) {$x$};
            \draw (a) -- (b) -- (d) -- (c);
        \end{tikzpicture}
    };

    % Right vertex Hasse diagram
    \node (right) at (3,0) {
        \begin{tikzpicture}
            \node[element] (a) at (0,0) {$w$};
            \node[element] (b) at (0,2.2) {$y$};
            \node[element] (c) at (0.8,1.5) {$z$};
            \node[element] (d) at (0.8,0.8) {$x$};
            \draw (a) -- (b) -- (d) -- (c);
        \end{tikzpicture}
    };

    % Bottom vertex Hasse diagram
    \node (bottom) at (0,-3) {
        \begin{tikzpicture}
            \node[element] (a) at (0,0.8) {$w$};
            \node[element] (b) at (0,2.2) {$y$};
            \node[element] (c) at (0.8,1.5) {$z$};
            \node[element] (d) at (0.8,0) {$x$};
            \draw (a) -- (b) -- (d) -- (c);
        \end{tikzpicture}
    };

    % Tail Hasse diagram
    \node (tail) at (0,-6) {
        \begin{tikzpicture}
            \node[element] (a) at (0,1.5) {$w$};
            \node[element] (b) at (0,2.2) {$y$};
            \node[element] (c) at (0.8,0.8) {$z$};
            \node[element] (d) at (0.8,0) {$x$};
            \draw (a) -- (b) -- (d) -- (c);
        \end{tikzpicture}
    };

    % Draw the edges of the diamond
    \draw (top) -- (left);
    \draw (top) -- (right);
    \draw (left) -- (bottom);
    \draw (right) -- (bottom);

    % Draw the tail connection
    \draw (bottom) -- (tail);
\end{tikzpicture}
\end{center}
\caption{Different linearizations of the poset $\pcauset[alt={x<y>w<z}]{2, 4, 1, 3}$.  }
\label{fig:hasseN}
\end{figure}

% \begin{figure}[h!]
%     \centering
%     \includegraphics[angle=90, width=.4\textwidth]{Nddraw.jpg}
%     \caption{Representation of the order polytope of the $N$ poset}
%     \label{fig:Nsimplex}
% \end{figure}
\end{example}

It is known (see for example Section 3.2 of~\cite{SP}) that two $n$-simplices in an order polytope share a $(n-1)$-face if the corresponding linearizations differ only on a pair $x_i<x_j$ and $x_j<x_i$, with the face being the region where $x_i=x_j$. In Figure~\ref{fig:hasseN}, two linearizations are connected if their corresponding simplices share a face.

The order polytope of a poset is a convex set (this follows from Lemma~\ref{L:Convp}). The following lemma relates the study of posets with the study of convex polytopes in $[0,1]^n$. In the proof of this lemma, we assume that any $n$-simplex is written using ``increasing coordinates'' (for example $\Delta[2]=\{(x_1,x_2)\ | \  0\leq x_1\leq x_2\leq 1\}$). 
Then, each $n$-simplex induces a fixed linear order on the set of subscripts $\{1,2,\dots,n\}$.

\begin{remark}
Another possible presentation in increasing coordinates of the $2$-simplex is 
 $\Delta[2]=\{(x_1,x_2)\ | \ 0\leq x_2\leq x_1\leq 1\}$, which induces the order $2<1$ on the set $\{1,2\}$. 
 We need both presentations to divide the square into a union of two 2-simplices.

 \end{remark}

\begin{lemma}\label{Lemma:1}
    Let $C$ be a convex $n$-dimensional subset of the unit $n$-cube. If $C$ is the union of $n$-simplices, all sharing the line from the zero vector to the one vector, then the convex set $C$ is an order polytope.

\end{lemma}
\begin{proof}
Every simplex of $C$ represents the order polytope of a linear order in the set $\{1,2,\dots,n\}$ that indexes the orthogonal basis of $\mathbb{R}^n$.  Let the order $P$ be the intersection of these linear orders. By definition, $C$ is contained in $\text{Poly}(P)$. It remains to show that $\text{Poly}(P)\subset C$.

We proceed to prove the lemma by contradiction. Among all the $n$-simplices in $\text{Poly}(P)\setminus{C}$, let $S$ be simplex that shares (at least) two faces ($(n-1)$-simplices) with the corresponding simplices in $C$. Then, by connecting the points on those two faces of $S$ with a line and by the convexity of $C$, we show that $S$ intersects $C$ at more points than the boundary, which is a contradiction.

Now, we know that simplices of $\text{Poly}(P)\setminus C$ cannot share more than one face with a simplex in $C$. Pick any simplex $T$ that shares exactly one face with a simplex in $C$. Then, for some pair of coordinates $t_i, t_j$, the points in the simplex $T$ satisfy $t_i<t_j$, but any other element of $C$ satisfies $t_j\leq t_i$. Thus, the simplex $T$ does not belong to $\text{Poly}(P)$ (as it introduces a relation $t_i<t_j$ not contained in $C$), which is a contradiction.

Finally, it is not possible that for every simplex of $\text{Poly}({P})\setminus C$, their intersection with  $C$ consists on a simplex of dimension at most $n-2$, as we explain below.

Take a point in the center of a simplex $U\in \text{Poly}(P)\setminus{C}$ and a point in the center of a simplex in $C$. The line between these points is included in $\text{Poly}(P)$. This line intersects the boundary of a simplex in $\text{Poly}(P)$ and a simplex in $C$ and, by varying the point in a ball inside of $U$, we find that the intersection is open and thus a whole face is shared by $\text{Poly}(P)$ and $C$.

We conclude that $C=\text{Poly}(P)$.
\end{proof}

\begin{corollary}\label{Cor:polytope}
    There is an assignment $ \mathrm{Poly}(P) \rightarrow P$.
\end{corollary}
\begin{proof}
The proof of Lemma~\ref{Lemma:1} and the results of Section 3.2 from~\cite{SP} give us a way to uniquely reconstruct an order out of the simplices of the order polytope.    
\end{proof}

In other words, the polytope $\text{Poly}(P)$ not only contains information about the underlying topological space, but we also know the order on each simplex; equivalently, we remember the poset that generated the polytope.

\subsubsection{Tropical algebra}

Consider the tropical semiring $\left(\mathbb{R}\sqcup\{-\infty\}, \oplus,\otimes\right)$, where for tropical elements $a$ and $b$, $a\oplus b =\max\{a,b\}$ and $a\otimes b=a+b$. Although some of the usual properties still hold, such as the distributive property $a\otimes(b\oplus c)=a\otimes b\oplus a\otimes c$,
tropical elements behave in surprising ways. For instance,
\begin{equation*}
\begin{aligned}
x\oplus x&=x, \\
1\otimes x&\neq x.
\end{aligned}
\end{equation*}

Tropical fractions are defined by $a\oslash b:=a-b$ and, using the fact that
$\max\{a-b,c\}=\max\{a,b+c\}-b$ and $b=\max\{b,0\}-\max\{-b,0\}$, one can show that for tropical numbers $a, b, a_2, b_2$:
\begin{equation*}
\begin{aligned}
(a\oslash b)\otimes( a_2\oslash b_2)&=(a\otimes a_2) \oslash (b\otimes b_2), \\
(a\oslash b) \oplus( a_2\oslash b_2)&=\big(a\otimes b_2\oplus a_2\otimes b \big)\oslash (b\otimes b_2).
\end{aligned}
\end{equation*}

\begin{remark}
A tropical polynomial $$\bigoplus_{i=1}^s c_i \otimes x_{1i}\otimes \cdots \otimes x_{ni}$$
corresponds to the function:
$$\max_{1\leq i\leq s}\{ c_i+x_{1i}+\cdots+x_{ni}\}.$$
\end{remark}

\begin{definition}
We say that a tropical polynomial $f(x)$ is in its \textnormal{reduced form} when there are no repeated (tropical) monomials in its expression.
\end{definition}

\begin{definition}
As defined in the work of~\cite{TG}, given a tropical polynomial $f(x)$ in its reduced form, the \textnormal{polytope of a tropical polynomial} $\mathrm{Poly}(f)$ is constructed by taking the convex envelope of 
the coordinate vectors from the monomials according to the following rule: given a tropical monomial $cx_1^{a_1},\dots x_n^{a_{n}}$, we associate the coordinate vector $(a_1,\dots,a_n,c)$.  
\end{definition}

This operation is invertible, associating to a convex polytope $C$ with positive integer coordinates (except perhaps the last one) its tropical polynomial $\text{Tr}(C)$.

\begin{example}\label{Ex:5}
By the above definition, we have the following relation:
$$4x^2y\oplus 1\leftrightarrow \{\lambda(2,1,4)+(1-\lambda)(0,0,1)\ | \ 0\leq \lambda\leq 1\}. $$
Here the monomial $4x^2y$ is related to the point $(2,1,4)$, and the monomial $1x^0y^0$ is related to the point $(0,0,1)$. 
\end{example}

 \begin{definition}
    Given a poset $P$, its \textnormal{tropical polynomial} $\mathrm{Tr}(P)$ is defined as the tropical sum of the tropical monomials defined by the vertices of $\mathrm{Poly}(P)\times 0\subset \mathbb{R}^{|P|+1}$.
\end{definition}

Since the order polytope of a poset is a convex set (see Lemma~\ref{L:Convp}), we can associate a tropical polynomial to each poset by using the vertices of its order polytope. However, we assume that there is an extra coordinate, set to $0$, ensuring that the values represent only the exponents of the monomials and not the constants. In Example~\ref{Ex:5}, the vector $(2,1,4)$ is associated with the monomial $4x^2y$, while the tropical polynomial $\text{Tr}(\pcauset[alt={1-\hbox{chain}}]{1})$ will be the tropical polynomial of $\text{Poly}(\pcauset[alt={1-\hbox{chain}}]{1})\times 0$
(where $0\oplus 0\otimes x=0\oplus x$). 

Then, the general expression of the tropical polynomial of a poset $P$ is of the form $\text{Tr}(P)=0\bigoplus x^v$, where $v$ belongs to the set of nonzero vertices of the order polytope of $P$.

\begin{example}
The tropical polynomial of the linear order, or $n$-chain, $\chain[n]$ admits a factorization of the form: 
\begin{equation}\label{Eqn:simplex}
  0\oplus x_n \otimes \Big(0 \oplus x_{n-1} \otimes \big( \cdots\otimes (0\oplus x_1)\big) \Big).  
\end{equation}

We denote $\mathrm{Tr}(\chain[n])$ by $\mathrm{Tr}(n)$. 
\end{example}
\begin{definition}\label{def:exp}
    We say that the tropical polynomial of a poset $P$ with $|P|=n$ is in its \textnormal{expanded presentation} if it is a tropical sum of tropical polynomials of $n$-chains, where every tropical polynomial of an $n$-chain appears only once.
\end{definition}

\begin{example}
Given the $\pcauset[alt={x<y>w<z}]{2, 4, 1, 3}$ poset, we list the vertices of $\mathrm{Poly}(\pcauset[alt={x<y>w<z}]{2, 4, 1, 3})$, where we assume a vector $(w,x,y,z)$, as follows: \begin{eqnarray*}
V(\mathrm{Poly}(\pcauset[alt={x<y>w<z}]{2, 4, 1, 3})) &=&\{(0,0,0,0),(1,1,1,1)\}
\\&\cup&\{(0,0,0,1),(0,0,1,1),(0,1,1,1)\}\ \ \hbox{from Equation~\eqref{eq1_}}\\
&\cup& \{(0,0,1,0),(0,0,1,1),(0,1,1,1)\} \ \ \hbox{from Equation~\eqref{eq2_}}\\
&\cup& \{(0,0,0,1),(0,0,1,1),(1,0,1,1)\} \ \ \hbox{from Equation~\eqref{eq3_}}\\
&\cup& \{(0,0,1,0),(1,0,1,0),(1,0,1,1)\} \ \ \hbox{from Equation~\eqref{eq4_}}\\
&\cup& \{(0,0,1,0),(0,0,1,1),(1,0,1,1)\} \ \ \hbox{from Equation~\eqref{eq5_}}\\
\end{eqnarray*}
Without repetition, we obtain the following vertices:
$$\{(0,0,0,0),(0,0,0,1),(0,0,1,0),(0,0,1,1),(1,0,1,0),(0,1,1,1),(1,0,1,1),(1,1,1,1)\}.$$
% \end{example}

% \begin{example}
The tropical polynomial in its expanded presentation associated to $\pcauset[alt={x<y>w<z}]{2, 4, 1, 3}$ is the following:
\begin{eqnarray}
    \mathrm{Tr}(\pcauset[alt={x<y>w<z}]{2, 4, 1, 3})(w,x,y,z)&=&(0\oplus z\oplus yz\oplus xyz\oplus wxyz)\nonumber\\
    &\bigoplus&(0\oplus y\oplus yz\oplus xyz\oplus wxyz)\nonumber\\
    &\bigoplus&(0\oplus z\oplus yz\oplus wyz\oplus wxyz)\nonumber\\
    &\bigoplus&(0\oplus y\oplus wy\oplus wyz\oplus wxyz)\nonumber\\
   &\bigoplus&(0\oplus y\oplus yz\oplus wyz\oplus wxyz).\label{N action}
\end{eqnarray}
\end{example}

Note that $x\oplus x=x$, so the expression above could be simplified as in Table~\ref{table:ex-PA_}. We prefer to represent the tropical polynomials of posets as tropical sums of tropical polynomials of linear orders.

\begin{table}[h]
    \centering
\begin{tabular}{ |c||c|c| } 
 \hline
 Poset $P$ & $\text{Tr}(P)$  \\\hline \hline 
\pcauset[alt={1-\hbox{chain}}]{1} & $0\oplus x$  \\ \hline 
\pcauset[alt={3-\hbox{chain}}]{1, 2, 3 } & $0\oplus z\oplus yz\oplus xyz$ \\\hline  
\pcauset[alt={x,y}]{2, 1} & $0\oplus x\oplus y\oplus xy$\\\hline 
\pcauset[alt={x<z, y<z}]{2, 1,3} & $0\oplus z\oplus xz\oplus yz\oplus xyz$\\\hline 
\pcauset[alt={1-\hbox{chain}U2-\hbox{chain} }]{3, 1, 2} & $0\oplus y\oplus z\oplus xy\oplus yz \oplus xyz$\\\hline 
\pcauset[alt={x,y,z }]{3, 2, 1} & $0\oplus y\oplus x\oplus z\oplus xy\oplus yz\oplus xz \oplus xyz$\\\hline 
\pcauset[alt={x<y>w<z}]{2, 4, 1, 3} & $0\oplus z\oplus y\oplus yz\oplus xy\oplus ywz\oplus xyz\oplus wxyz$\\
 \hline
\end{tabular}
    \caption{Examples of poset tropical polynomials}
    \label{table:ex-PA_}
\end{table}

The tropical polynomials of posets are determined by the tropical polynomials of each linearization, which correspond to the tropical polynomials of a set of  maximal simplices of the order polytope of the poset (given a polytope, we call maximal simplices those of highest dimension). The assignment from the set of linearizations of the poset to the set of maximal simplices in the order polytope of a poset $P$ is injective (see Lemma~\ref{Lemma:1}).

\begin{lemma}
    The map $P\rightarrow \mathrm{Tr}(P)$ from finite posets to tropical polynomials in their expanded presentation, is injective.\label{Lemma:dist}
\end{lemma}    
\begin{proof}

Let $f(x):=\text{Tr}(P)$ be a tropical polynomial of a poset $P$. Then, it is possible to determine the tropical polynomials of the linearizations of $P$ from the expression of $f(x)$, as follows: each maximal simplex has $k=|P|+1$ vertices, when transforming those vertices into monomials we obtain a sequence of $k-1$ monomials, each dividing the next one, and zero (corresponding to the vector zero in the order polytope). 

Then, from a tropical polynomial of a poset we can recover the linear orders of the poset, by searching for chains of monomials where the $i+1$-th term divides the $i$-th term,
and since a poset is determined by the intersection of its linearizations, we also recover the original poset.
\end{proof}

\begin{remark}
 For a given poset $P$, there is a well-known polynomial $\Omega(P)$, known as the order polynomial, which was introduced by~\cite{Chromatic}. When two posets have the same order polynomial, such as $\pcauset[alt={x<y,w<z }]{3, 4, 1, 2}$ and $\pcauset[alt={x<y, x<w, x<z }]{1,4,3,2}$ (since $\Omega(\pcauset[alt={x<y,w<z }]{3, 4, 1, 2})=\Omega(\pcauset[alt={x<y, x<w, x<z }]{1,4,3,2})$), they are called Doppelg\"anger (\cite{Doppelganger}).
 The question  of whether a set of polynomials exists that is capable of distinguishing posets is addressed by Lemma~\ref{Lemma:dist}.
\end{remark}

\subsubsection{Neural networks and tropical polynomials}

Up to this point, we have linked each poset to its associated tropical polynomial. Next, we will describe the neural networks associated with tropical polynomials of posets.

First, we explore tropical operations at the level of neural networks. 

\begin{remark}
    In the context of IVNN, nodes in the same layer, say $(y_1,\dots,y_n)$, can have a different parameter $t$ in their $\mathrm{ReLU}$ activation function, that is, 
    \begin{equation*}
    \begin{aligned}
        \mathrm{ReLU}_{(t_1,\dots,t_n)}(y_1,\dots,y_n):&=\mathrm{ReLU}((y_1,\dots,y_n),{(t_1,\dots,t_n)})\\
        &=(\max\{y_1,t_1\},\dots,\max\{y_n,t_n\}).
        \end{aligned}
    \end{equation*}
    
\end{remark}
\begin{definition}
If the neural network $M$
%$a$
corresponds to the tropical polynomial  $f(x)$ and $N$
%$b$
corresponds to $g(x)$, then from: 
\begin{eqnarray}
f\otimes g &=&f+g=max(f+g,-\infty),
\end{eqnarray}
we define at the level of neural networks the \textnormal{tropical product} as the following:
\begin{eqnarray}
M\otimes N &=&\mathrm{ReLU}\left(\begin{bmatrix}
M&
N
\end{bmatrix}\begin{bmatrix}
1\\
1\\
\end{bmatrix},-\infty\right).
\end{eqnarray}

\end{definition}

\begin{definition}\label{def:frac}
We define the \textnormal{tropical fraction} $M\oslash N$ of two neural networks $M$ and $N$ as the following operation:

\begin{eqnarray}
M\oslash N &=&\mathrm{ReLU}\left(\begin{bmatrix}
M&
N
\end{bmatrix}\begin{bmatrix}
1\\
-1\\
\end{bmatrix},-\infty\right).
\end{eqnarray}
\end{definition}

We note that the tropical sum has the following properties:

\begin{eqnarray}
f\oplus g &=& \max\{f,g\}\nonumber\\
&=& \max\{f-g,0\}+g\nonumber\\
&=&\max\{f-g,0\}\otimes\max\{g,0\}\oslash\max\{-g,0\}.\label{choice1}
\end{eqnarray}

\begin{definition}
By computing the vector $[\max\{f-g,0\}\otimes\max\{g,0\}, \ \max\{-g,0\}]$, we define the \textnormal{tropical sum} of two neural networks $a$ and $b$ as:

\begin{equation*}
M\oplus N=\mathrm{ReLU}\left(\mathrm{ReLU}\left(\begin{bmatrix}
M&
N\end{bmatrix}
\begin{bmatrix}
1&0&0\\
-1&1&-1\end{bmatrix},0\right)
\begin{bmatrix}
1\\
1\\
-1\end{bmatrix},-\infty\right).
\end{equation*}
\end{definition}

\begin{remark}
This means that we have made choices, such as taking $\max\{f-g\}$ instead of $\max\{g-f\}$, to determine the factorization of $f\oplus g$ in terms of the usual ReLUs with parameters $0$ or $-\infty$.\label{rm2}
Thus, \textbf{more than one neural network can be assigned to a tropical polynomial.}

If we make the other choice and describe the tropical sum as:
\begin{eqnarray}
f\oplus g &=& \max\{f,g\}\nonumber\\
&=& \max\{g-f,0\}+f\nonumber\\
&=&\max\{g-f,0\}\otimes\max\{f,0\}\oslash\max\{-f,0\}\nonumber\\
&=&\max\left\{ \max\{g-f,0\}\otimes\max\{f,0\}\oslash\max\{-f,0\},-\infty \right\}\label{choice2},
\end{eqnarray}
then the tropical sum $M\oplus N$ is as follows:
\begin{equation*}
M\oplus N=\mathrm{ReLU}\Bigg(\mathrm{ReLU}\left(\begin{bmatrix}
M&
N\end{bmatrix}
\begin{bmatrix}
-1&1&-1\\
1&0&0\end{bmatrix},0\right)
\begin{bmatrix}
1\\
1\\
-1\end{bmatrix},-\infty\Bigg).\\
\end{equation*}

This way we have obtained two different neural network architectures associated to the same neural network.
\end{remark}

\begin{remark}\label{remark:several}
If a neural network is the result of a binary operation between two neural networks with a different number of layers, we can enforce the same number of layers by applying $\mathrm{ReLU}_{-\infty}$ and the multiplication by the identity matrix (either at the beginning, at the end, or in the middle) to the input that has fewer layers. 
\end{remark}

Now that we have defined the basic tropical operations on neural networks, we can move on to defining the neural network associated to a simplex (or $n$-chain). We fix the presentation of the tropical polynomial of a chain by the factorization given in Equation~\eqref{Eqn:simplex}.

\begin{remark}
Since the polytope $\mathrm{Poly}(P)$ of a poset $P$ contains the origin, $\mathrm{Tr}(P)$ always has a zero, and thus the final layer of a poset neural network is $\mathrm{ReLU}_0$. 

\end{remark}
\begin{definition}
Let $I_{1,n-1}$ be the $(n-1)$-dimensional identity matrix with the first row repeated. Given a vector $v=(v_1,\dots,v_n)$ of dimension $n$, we define  $$\mathrm{ReLU}_{0,-\infty,\dots,-\infty}(v)=(\mathrm{ReLU}_0(v_1), \mathrm{ReLU}_{-\infty}(v_2),\dots,\mathrm{ReLU}_{-\infty}(v_n)).$$ 
If $A$ is a $(m\times n)$-dimensional matrix with rows $\{R_i\}_{1\leq i\leq m}$, then $\mathrm{ReLU}_{0,-\infty,\dots,-\infty}(A)$ is the matrix  with rows $\{\mathrm{ReLU}_{0,-\infty,\dots,-\infty}(R_i)\}_{1\leq i\leq m}.$
\end{definition}

\begin{lemma}\label{lemma:simpl}

    The neural network associated to the tropical polynomial of a chain is an iteration of transformations of the form: $$X_i\mapsto Y_i= \mathrm{ReLU}_{0,-\infty,\dots,-\infty}(X_i)$$
    and $$Y_i\mapsto X_{i-1}=X_iI_{1,i-1},$$
    where each $X_i$ is a $(1\times i)$-dimensional vector and where $X_n$ is the input.
\end{lemma}\begin{proof}
    It follows from direct computation and Equation~\eqref{Eqn:simplex}.
\end{proof}

Row 4 of Table~\ref{fig:algebra1} and Table~\ref{fig:algebra2} show the resulting neural networks of the two chain $\{x<y\}$ and the three chain $\{x<y<z\}$, respectively.

\begin{table}[h]
\centering
    \caption{Tropical polynomial, polytope, and IVNN of the two chain.}
    \label{fig:algebra1}
\begin{tabular}{|l||l|}
\hline
Poset & $\{x<y\}$ \\
\hline
Tropical & $0\oplus y\oplus x\otimes y$ \\
 polynomial & $=0\oplus y\otimes (0\oplus x)$\\
\hline 
Polytope & $\{0\leq x\leq y\leq 1\}$ \\

\hline
IVNN & ${\text{ReLU}_0\left(\begin{bmatrix}
\text{ReLU}_0(x)&
\text{ReLU}_{-\infty}(y)
\end{bmatrix}\begin{bmatrix}
1\\
1\\
\end{bmatrix}\right)}$ \\
\hline 
\end{tabular}
\end{table}

\begin{table}[h]
    \caption{Tropical polynomial, polytope, and IVNN of the three chain.}
    \label{fig:algebra2}
\begin{tabular}{|l||l|}
\hline
Poset & $\{x<y<z\}$ \\
\hline
Tropical & $(0\oplus z\oplus yz\oplus xyz)$ \\
 polynomial & $=(0\oplus z\otimes (0\oplus y\otimes(0\oplus x)))$\\
\hline 
Polytope & $\{0\leq x\leq y\leq z\leq 1\}$ \\

\hline
IVNN & ${\text{ReLU}_0\left(\text{ReLU}_{(0,-\infty)}\left(\begin{bmatrix}
\text{ReLU}_0(x)&\text{ReLU}_{-\infty}(y)&\text{ReLU}_{-\infty}(z)
\end{bmatrix} 
\begin{bmatrix}
1&0\\
1&0\\
0&1
\end{bmatrix}\right)
\begin{bmatrix}
1\\
1\\
\end{bmatrix}\right)
}$
    \\

&

\begin{tikzpicture}[x=2.1cm,y=1.2cm]
  \large
  \message{^^JNeural network without arrows}
  \readlist\Nnod{3,2,1} % array of number of nodes per layer
  
  % TRAPEZIA
  
  \message{^^J  Layer}
  \foreachitem \N \in \Nnod{ % loop over layers
    \def\lay{\Ncnt} % alias of index of current layer
    \pgfmathsetmacro\prev{int(\Ncnt-1)} % number of previous layer
    \message{\lay,}
    \foreach \i [evaluate={\y=\N/2-\i+0.5; \x=\lay; \n=\nstyle;}] in {1,...,\N}{ % loop over nodes
      
      % NODES
      \node[node \n,outer sep=0.6] (N\lay-\i) at (\x,\y) {};
      
      % CONNECTIONS
      \ifnum\lay>1 % connect to previous layer
        \foreach \j in {1,...,\Nnod[\prev]}{ % loop over nodes in previous layer
          \draw[connect,white,line width=1.2] (N\prev-\j) -- (N\lay-\i);
          \draw[connect] (N\prev-\j) -- (N\lay-\i);
          %\draw[connect] (N\prev-\j.0) -- (N\lay-\i.180); % connect to left
        }
      \fi % else: nothing to connect first layer
      
    }
  }
  
  % LABELS
  \node[above=0.3,align=center,green!60!black!60!black] at (N1-1.90) {input};
  \node[above=0.5,align=center,red!80!black!60!black] at (N\Nnodlen-1.90) {output};
  
\end{tikzpicture}\\
\hline 

\end{tabular}
\end{table}

A tropical polynomial can have several neural networks associated to it by Remark~\ref{rm2} and Remark~\ref{remark:several}. Poset tropical polynomials (in extended form) are the sum of tropical polynomials of linearizations, and all tropical polynomials of linearizations have the same number of monomials.

We choose a neural network representative of each tropical polynomial of a poset in a coherent way by using Lemma~\ref{lemma:simpl}, which describes the neural network associated to a linear order.

\begin{definition}
    Following~\cite{AlexNEt} and \cite{LeNEt}, we define an \textnormal{inception module} as a network that takes an input vector $x$ and sends it to $k$ parallel running neural networks, each of which returns a coordinate of the $k$-dimensional output of the inception module.
\end{definition}

It is worth mentioning that the inception module is defined in an abstract manner, but a physical implementation requires ordering the linearizations of the poset, or neural networks, that form the inception module. 
 \begin{definition}\label{Def:NN}
 Let $P$ be a poset with $k$ linear orders. Then, a \textnormal{poset neural network} is defined as the inception module formed by the $k$ neural networks associated with each linear order of $P$. 
\end{definition}

\begin{example}
    
The poset $\{x,y\}$ has two linearizations: $x<y$ and $y<x$. For each one of them, we construct their neural network as explained in Lemma~\ref{lemma:simpl}, and the final neural network is the coordinate wise concatenation of the neural networks of the two linearizations. The function that maps ${\begin{bmatrix}
   x  \\
   y  \\
  \end{bmatrix} }$ to ${\begin{bmatrix}
   x & y \\
   y & x \\
  \end{bmatrix} }$ allows us to take the input $(x,y)$ and send copies of it, each of which can then be permuted in a different way. 

\end{example}

  \begin{remark}(Broadcasting of the NN architecture)
Given a poset $P$, the neural networks associated to its linearizations have the same architecture. 
Effectively, a poset neural network consists of a nonlinear function that takes a $(n\times 1)$-dimensional vector and returns a $(n\times 1\times k)$-dimensional matrix whose columns (in the $k$ direction) are certain permutations of the input, followed by the architecture of Lemma~\ref{lemma:simpl} applied to the $(n\times 1\times k)$-matrix previously constructed.       
       See for example Table~\ref{fig:algebra11}.

  \end{remark}
\begin{table}[h]
\centering
    \caption{Tropical polynomial, polytope, and IVNN of the disjoint union of points.}
    \label{fig:algebra11}
\begin{tabular}{|l||l|}
\hline
Poset & $\{x,y\}$ \\
\hline
Tropical & $0\oplus y\oplus x \oplus x\otimes y$ \\
 polynomial &$=\left(0\oplus y\otimes(0\oplus x)\right) \oplus \left(0\oplus x\otimes(0\oplus y)\right) $ \\
\hline 
Polytope & $\{0\leq x\leq 1, 0\leq y\leq 1\}$ \\

\hline
IVNN & ${\text{ReLU}_0\left(\begin{bmatrix}
\text{ReLU}_0(x)&
\text{ReLU}_{-\infty}(y)\\
\text{ReLU}_0(y)&
\text{ReLU}_{-\infty}(x)
\end{bmatrix}\begin{bmatrix}
1\\
1\\
\end{bmatrix}\right)}$ \\
\hline 
\end{tabular}
\end{table}

Since a poset uniquely specifies its linearizations, the neural network of a poset is well defined up to the order of the networks in the inception module or the output of the permutation function.

We may ask ourselves what properties of posets are inherited by poset neural networks. Given an industrial problem, the first step when trying to solve it with neural networks is to review the existing literature to determine whether something similar has been addressed before. As part of the efforts to clarify whether a neural network is capable of solving a classification problem, we note the following corollary.

\begin{corollary}\label{C:v}
    The number of vertices in the order polytope of a poset is an upper bound for the number of linear regions of the poset neural network.
\end{corollary}
\begin{proof}
Since poset neural networks are IVNN, and the order polytope is essentially the Newton polygon of the corresponding tropical polynomial, the result follows from~Section 3 of~\cite{TG} and Proposition 3.1.6 of \cite{Intro},    
\end{proof}
  This number is called \textit{geometric complexity}, as seen in~\cite{TG}.
  \begin{corollary}\label{C:layers}
      Given a poset $P$ and its poset neural network $M$, the number of points of the poset $P$ is the number of layers (linear transformation + nonlinearity) of the neural network $M$,
\end{corollary}
  \begin{proof}
      It follows from Definition~\ref{Def:NN} and Lemma~\ref{lemma:simpl}.
  \end{proof}
  
Note that a poset neural network could be expressed in different ways.
Our representation (which relies on Equation~\eqref{Eqn:simplex}) preserves the geometric information according to Corollary~\ref{C:v}  and Corollary~\ref{C:layers}.

\subsection{Neural network architectures}
In a calculator, an operation can be evaluated by first constructing an Abstract Syntax Tree (AST). The evaluation proceeds from top to bottom, applying the corresponding elementary operations that label the vertices of the tree. Motivated by abstract syntax trees, but replacing the tree by its poset of vertices, we will use finite posets to assemble neural networks. 

When describing the architecture of a neural network, for instance a feed forward neural network, we usually describe the order in which we concatenate the matrix products and the non-linearities. Poset neural networks admit an alternative description. This alternative approach has the advantage of being interpretable from a geometric view point, but it relies on category theory and operad theory. Due to the strong mathematical requirements, we leave the details to Section~\ref{Sec:end}.   

Here is a brief overview of our results. Instead of blocks of matrix multiplications and non linearities, our elementary operations will be parametrized by posets' operations (see Definition~\ref{mainDef}). The motivating example is the operation $\pcauset[alt={x,y}]{2, 1}(N,M)$. For a given pair $N$ and $M$ of IVNN, and their corresponding polytopes $\text{Poly}(N)$ and $\text{Poly}(M)$, the operation $\pcauset[alt={x,y}]{2, 1}$ admits the following geometric interpretation: $\text{Poly}(\pcauset[alt={x,y}]{2, 1}(N,M))$   coincides with the Minkowski sum of $\text{Poly}(N)$ and $\text{Poly}(M)$. We will provide a similar interpretation for each operation described by a poset.

In particular, if we have a poset $P$ parametrizing an operation, and we evaluate the poset on the trivial inputs $x_1,\dots,x_n$, (with trivial input we mean a ($1\times1$)-dimensional identity matrix), then $P(x_1,x_2,\dots,x_n)$ is a neural network and we can think of $P$ as the architecture of the neural network on the variables $x_1,\dots,x_n$. We also have compositionality, that is, given $P_1(x_{1,1},\dots,x_{1,i_1}),\dots,P_n(x_{n,1},\dots,x_{n,i_n}),$ and a poset $Q$ with $n$ points, we can define the composition neural network
$$Q(P_1(x_{1,1},\dots,x_{1,i_1}),\dots,P_n(x_{n,1},\dots,x_{n,i_n})).$$

The paper by~\cite{TG} describes explicitly several polytopes associated to a neural network: for every $k$ there is a polytope associated to the first $k$-layers of the neural network.
In our setting, the polytopes of the neural networks $\{P_j(x_{1,1},\dots,x_{j,i_j})\}$ are transformed into the polytope of $Q(P_1(x_{1,1},\dots,x_{1,i_1}),\dots,P_n(x_{n,1},\dots,x_{n,i_n}))$ by means of a polytope operation that depends on $Q$.

Currently, we can only formalize the previous steps for poset neural networks. We describe poset operations on tropical polynomials in Definition~\ref{TPA}, and from Equation~\eqref{def:frac} we can associate to any pair of posets $P,Q$ a tropical rational function and its corresponding neural network. However, the extension to all IVNN is still open.
The difficulty lies on the fact that more than one IVNN can be associated to a tropical polynomial.

\section{Poset Pooling Filters}
\label{Sec:CFilters}

Using the results of Section \ref{F:section}, we proceed to define a new family of pooling filters. The max pooling operation, introduced by~\cite{max}, returns the maximum value in a certain region, but is well known to cause a loss of information, specifically when there are several large values, as the pooling only selects one and ignores the other values. The average pooling filter, a second popular filter given by~\cite{average}, takes the average of inputs in a certain region, but has as a drawback that if one value is large and the remaining ones are small, the output is close to the larger value divided by four, which in some situations may not be desirable. To address these and other issues, a third filter was suggested by~\cite{Mixedpool}: mixed pooling, in which average pooling and max pooling are applied randomly.

Now, looking at the backpropagation step for the average pooling filter, the gradient from the next layer is equally distributed among all entries entering the pooling layer (irrespective of whether they were zero or nonzero during the forward pass). In the case of the max pooling, if there are several entries that achieve the maximum or are close, max pooling passes the gradient to only one of them. Mixed pooling then either equally distributes the gradient or may not pass the gradient to relevant entries.

We propose a family of filters that are more precise during the backpropagation step.
\begin{definition}
For any poset $P$ with four points, we write its tropical polynomial in its simplified form (using $x\oplus x=x$). This expression represents a function which is a combination of the function $\max$ evaluated on a linear combination of variables with coefficients 0 or 1. We call this function the \textnormal{poset filter} of $P$.    
\end{definition}

\begin{definition}
For a given input $\begin{bmatrix}
   a_{0,0} & a_{0,1} \\
   a_{1,0} & a_{1,1} \\
  \end{bmatrix}$, the corresponding poset filter of the disjoint union of four points, whose order polytope is the cube, computes all possible sums out of the four inputs, as seen in Equation~\eqref{eqn:cube2}. 
\begin{equation}
\max\left\{
\begin{aligned}
   & \ \ \ \ 0, \\
   & \ \ \ \max_{i,j}\{a_{i,j}\}, \\
   & \max_{\substack{i,j,k,l\\(i,j)\neq (k,l)} }\{a_{i,j}+a_{k,l}\}, \\
   & \max_{\substack{i,j,k,l,m,n\\
(i,j)\neq (k,l)\\(i,j)\neq (m,n)\\(k,l)\neq (m,n)
}}\{a_{i,j}+a_{k,l}+a_{m,n}\}, \\
& \ \ \ \ a_{0,0}+a_{1,0}+a_{0,1}+a_{1,1}\\
\end{aligned} 
\right\}.
\label{eqn:cube2}
\end{equation}
\end{definition}

\begin{remark}
Let us have inputs of the form $a_{00}=-1$, $a_{01}=0$, $a_{10}=1.9$, and $a_{11}=2$ and assume that the output of the poset filter is $0a_{00}+0a_{01}+1a_{10}+1a_{11}$. During backpropagation, the incoming gradient is propagated back to the input values, with the gradient being scaled by the weights $(0,0,1,1)$. Notably, if multiple input entries achieve the maximum value, they all contribute to passing the gradient. This behavior differs from average pooling, where the gradient is evenly spread, and from max pooling, where only the maximum entry (in this case $a_{11}$) receives the gradient.

\end{remark}

An immediate drawback of the filter associated with the disjoint union of points is the number of operations. 
We consider then the filter associated to the four chain and the filter of the poset $\pcauset[alt={x<y>w<z}]{2, 4, 1, 3}$, from Table~\ref{table:ex-PA_}.

\begin{definition}
For a given input $\begin{bmatrix}
   a_{0,0} & a_{0,1} \\
   a_{1,0} & a_{1,1} \\
  \end{bmatrix}$, the corresponding filter of the four chain, whose order polytope is the four-simplex, is as follows: 
\begin{equation}
  \max\{0,\ a_{1,1},\ a_{1,1}+a_{1,0},\ a_{1,1}+a_{1,0}+a_{0,1},\ a_{1,1}+a_{1,0}+a_{0,1}+a_{0,0}\},  \label{eqn:simplex2}
\end{equation}
\end{definition}

\begin{definition}
For a given input $\begin{bmatrix}
   a_{0,0} & a_{0,1} \\
   a_{1,0} & a_{1,1} \\
  \end{bmatrix}$, the corresponding filter of the poset $\pcauset[alt={x<y>w<z}]{2, 4, 1, 3}$ is defined as follows:
\begin{equation}\max\left\{\begin{array}{c}
   0,\ a_{1,1},\ a_{1,0},\ a_{1,0}+a_{1,1},\ a_{0,0}+a_{1,0},\  a_{1,0}+a_{0,1}+a_{1,1},\\
   a_{0,0}+a_{1,0}+a_{1,1},\ a_{0,0}+a_{0,1}+a_{1,0}+a_{1,1}\\
  \end{array} \right\}.\label{eqn:N}\end{equation}
\end{definition}

Our experiments, detailed in Section~\ref{exp:q}, show that the simplex filter (or other $4$-point posets) can be used instead of the filter of the disjoint union of points with similar performance and fewer computations.

We also performed experiments with random vectors. Table~\ref{table:exp111} shows the results of experiments when training a neural network with random vectors instead of vectors given by posets. Although the average precision is lower than our best result, for this specific neural network the accuracy when using random vectors is still higher than the accuracy with average pooling, mixed pooling, and max pooling.

\section{Review of experimental results}\label{Sec:rev}
\subsection{Quaternion Neural Network}\label{exp:q}
With respect to the CIFAR10 dataset, we have first tested on a quaternion convolutional neural network. We aimed to conduct experiments by placing the poset filters in different locations within the convolutional part of the algorithm, and the quaternion architecture was convenient for this purpose, as the inputs at different layers have dimensions divisible by $2$, which simplified the early experiments. We worked on the architecture given in \url{https://github.com/JorisWeeda/Quaternion-Convolutional-Neural-Networks/tree/main} (called QCNN (Weeda impl.) on the results' tables), due to it being an implementation on Pytorch that uses the correct weight initialization that appears on the work of~\cite{oQCNN}.

In our experiments (Section~\ref{Sec:exp}), we have evaluated testing accuracy after placing the poset convolutional filter at several positions. Our best result (see Table~\ref{table:exp14}) was obtained by placing the $\pcauset[alt={x<y>w<z}]{2, 4, 1, 3}$ poset filter between the last two convolutions, just after a ReLU activation. We reduced the number of trainable parameters from $2,032,650$ to $656,394$, with a testing accuracy average of $78.92\%$, compared to the reported original testing accuracy average of $78.14\%$. This average was taken over 14 independent trainings. We also note that when reproducing the Weeda implementation, we got an average of $77.83\%$ under the same conditions as in our test (mainly, the addition of reproducibility seeds). The performance of our poset filter is superior to the alternative architecture obtained by replacing it with max pooling ($76.75\%$), average pooling ($76.097\%$) and mixed pooling ($76.41\%$). 
We report similar performance during cross-validation in Table~\ref{table:exp16}.

During these experiments, we were concerned that the presence of a ReLU function before the filters would affect the performance. However, we found that poset filters perform similarly whether they are preceded by a ReLU or not. We propose the following explanation: negative values decrease the value of the average.  When we evaluate the poset filter, it takes into account those negative values and, according to the poset, the exact linear combination that only adds positive values may not be part of the expression of the poset filter. By pre-filtering with a ReLU, the output value of the poset filter is possibly bigger. 

The poset disjoint union of points, with order polytope the ``cube'', has $16$ terms, making it slow. We run an experiment that evaluates all $16$ possible posets with four points. The results of Table~\ref{table:exp110} show that the accuracy does not differ too much, so the reader can choose to use, instead of the disjoint union poset, the $\pcauset[alt={x<y>w<z}]{2, 4, 1, 3}$ poset or the $4$ chain poset filter, with less computational time.

\subsection{CNN}
With respect to the Fashion MNIST dataset, reported in Section~\ref{Sec:ex3}, we tested on the implementation available at \url{https://github.com/Abhi-H/CNN-with-Fashion-MNIST-dataset}, which uses a standard convolutional neural network with ELU (\cite{ELU}) as the non-linearity and two max pooling functions. The original network achieves $90.65\%$ accuracy. Our best result, obtained by the network that replaces the second max pooling by the $\pcauset[alt={x<y>w<z}]{2, 4, 1, 3}$ filter, was an accuracy of $91.84\%$, but among $14$ independent training runs the average accuracy was $91.26\%$ (see Table~\ref{table:expCNN}).

\subsection{SimpleNet}
In Section~\ref{sec:simp}, we report experiments with SimpleNet, introduced by~\cite{simple}. We trained the architecture on a subset of the dataset ImageNet created by~\cite{imagenet}. Following~\cite{imagenet100pytorch}, we randomly selected 100 classes of Imagenet. 

 Since, unfortunately, adding a filter on the convolutional section would require major changes to the architecture, we chose then to replace the different instances of max pooling in the architecture by the $\pcauset[alt={x<y>w<z}]{2, 4, 1, 3}$ poset filter.

ImageNet's test set labels are not public. To compare the performance, we computed the validation accuracy at the time in which the validation loss is the lowest.
 Due to the substantial computational resources required, we report the average validation accuracy over $6$ repetitions for each architecture, instead of the originally planned $14$. Additionally, we do not include experiments with max pooling, average pooling, and mix pooling in this submission. These experiments are still in progress, and we plan to share the full results once they are complete.

\subsection{DenseNet}
  We implemented DenseNet (\cite{DenseNET_p}) with crop production set to $32\times 32$. We performed experiments on Fashion MNIST, noting that the reported performance of DenseNet on it is $95.4\%$. In our implementation of DenseNet we obtained an average performance of $95.2593\%$.

 Due to similar constraints as in the SimpleNet case, we chose to replace  different instances of average pooling in the architecture by the $\pcauset[alt={x<y>w<z}]{2, 4, 1, 3}$ poset filter. Our best result returned an average accuracy of $95.28\%$ with std $0.13$ over $14$ repetitions. While we obtained a higher accuracy than the original architecture, replacing the $\pcauset[alt={x<y>w<z}]{2, 4, 1, 3}$ poset filter by  max pooling increased the accuracy even more. 
 See Section~\ref{sec:Dense} for more details.

We also tested DenseNet with CIFAR10 and CIFAR100, but in all cases the accuracy did not improve significantly when substituting the existing pooling methods in the architecture with the $\pcauset[alt={x<y>w<z}]{2, 4, 1, 3}$ poset filter.

\subsection{Future experiments}
\label{Sec:conc}
We showed that for certain datasets and architectures, poset filters outperform average pooling, max pooling, and mixed pooling. We will continue our experiments with other possible architectures, tasks, and datasets.

\section{Operadic order theory and machine learning}\label{Sec:end}

In Definition~\ref{Def:NN} we introduced a set of IVNN parameterized by posets. In this section, we lift the structure of an algebra over an operad of posets from order polytopes to poset neural networks. This implies that we have a set of neural networks indexed by arbitrary finite posets and 
%that we have 
a family of associative operations on those neural networks indexed by posets.

\subsection{The operad of posets}

We first study a family of poset endomorphisms, that is, functions of the type $\text{posets}^n\rightarrow \text{posets}$.

\begin{definition}
    Let $P=\{x_1,\dots,x_n|<_P\}$ be a poset with $n$ points. The \textnormal{lexicographic sum} of $n$ posets $P_1,\dots, P_n$ along $P$ is defined as the poset $P(P_1,\dots, P_n)$, with points $ \cup P_i $ and 
    $$ x<_{P(P_1,\dots, P_n)}y\hbox{ if }\begin{cases}
          x,y\in P_i&\hbox{and } x<_{P_i}y, 
        \\
          x\in P_i, y\in P_j& \hbox{and } x_i<x_j.
    \end{cases} $$
\end{definition}

\begin{remark}
 The disjoint union of posets $P_1\sqcup P_2$ is the lexicogaphic sum along $\pcauset[alt={x,y}]{2, 1}$, and the ordinal sum, or concatenation, $P_1*P_2$ is the lexicographic sum along $\pcauset[alt={\hbox{disjoint union of two points}}]{1,2}$.
   
\end{remark}

 \begin{definition}
A \textnormal{series parallel poset} is generated by a point $\chain[1]=\pcauset[alt={1-\hbox{chain} }]{1}$ under the operations of disjoint union $\pcauset[alt={x,y}]{2, 1}$ and concatenation $\pcauset[alt={2-\hbox{chain} }]{1, 2}$.
Under the operation $\pcauset[alt={x,y}]{2, 1}=\sqcup$, series parallel posets form a commutative monoid with the following cancelation property: $P\sqcup Q=R\sqcup Q$ implies $P=R$ .
     
 \end{definition}

\begin{example}
  Examples of series parallel posets are any chain $\pcauset[alt={\hbox{disjoint union of two points}}]{1,2}=\pcauset[alt={\hbox{disjoint union of two points}}]{1,2}(\pcauset[alt={1-\hbox{chain}}]{1},\pcauset[alt={1-\hbox{chain}}]{1})$, and any tree $\pcauset[alt={x<y,x<z}]{1,3,2}=\pcauset[alt={\hbox{disjoint union of two points}}]{1,2}(\pcauset[alt={1-\hbox{chain}}]{1},\pcauset[alt={x,y}]{2, 1})$.
  \end{example}

\begin{definition}
Given a poset that is not a linear order, we say that it is \textnormal{indecomposable} if the only lexicographic sum the poset admits are the trivial lexicographic sums $P=\pcauset[alt={1-\hbox{chain}}]{1}(P)$ or $P=P(\pcauset[alt={1-\hbox{chain}}]{1},\dots,\pcauset[alt={1-\hbox{chain}}]{1})$. A \textnormal{factorization} for a poset is a decomposition into lexicographic sums evaluated on indecomposable posets. 
\end{definition}

\begin{example}
A factorization of $\pcauset[alt={x<y, x<z }]{1, 3, 2}$ is $\pcauset[alt={\hbox{disjoint union of two points}}]{1,2}(\pcauset[alt={1-\hbox{chain}}]{1}, \pcauset[alt={x,y}]{2, 1}(\pcauset[alt={1-\hbox{chain}}]{1},\pcauset[alt={1-\hbox{chain}}]{1}))$.
\end{example}

Note that linear orders admit more than one factorization.

\subsection{Endomorphisms of polytopes}\label{end:poly}
Let $P_1,\dots, P_n$ be posets. To define operations on the order polytopes $\text{Poly}(P_1),$ $\dots,$ $\text{Poly}(P_n)$, we work in $[0,1]^{\sum_{i=1}^n |P_i|}$ and assume that each $\text{Poly}(P_i)$ is embedded as $\overrightarrow{0}^{\sum_{j=1}^{i-1} |P_j|}\times \text{Poly}(P_i)\times \overrightarrow{0}^{\sum_{k=i+1}^n |P_k|}$.

\begin{definition}
At the level of polytopes, the \textnormal{Minkowski sum} takes $\mathrm{Poly}(P)$ and $\mathrm{Poly}(Q)$ and returns: \[\mathrm{Poly}(P)+ \mathrm{Poly}(Q)=\{x+y,\ x\in \mathrm{Poly}(P), \ y\in \mathrm{Poly}(Q)\}.\]
\end{definition}

\begin{lemma}\label{L:M} For two posets $P$ and $Q$ and their order polytopes $\mathrm{Poly}(P)$ and $\mathrm{Poly}(Q)$,
    \[\mathrm{Poly}(\pcauset[alt={x,y}]{2, 1}(P,Q))=\mathrm{Poly}(P)+ \mathrm{Poly}(Q).\]
\end{lemma}
\begin{proof}
Given two posets $P$ and $Q$, $P\sqcup Q$ is the poset with no new order relations.
Then, $\text{Poly}(P\sqcup Q)$ has as many orthogonal coordinates as the points in $P$ and $Q$, and the relations of $P$ do not interfere with those of $Q$. Thus, the points of $\text{Poly}(P,Q)$ are of the form $\{(x,y)|\ x\in \text{Poly}(P), \ y\in \text{Poly}(Q)\}$ and $\text{Poly}(P,Q)=\text{Poly}(P)\times \overrightarrow{0}^{|Q|}+\overrightarrow{0}^{|P|}\times \text{Poly}(Q)$, where the sum of polytopes is given by the Minkowski sum.

\end{proof}
Another operation on polytopes is the convex envelope.

\begin{definition}
The \textnormal{convex envelope} of $\mathrm{Poly}(P)$ and $\mathrm{Poly}(Q)$, or $C(\mathrm{Poly}(P),\\ \mathrm{Poly}(Q))$, is defined as:
$$C(\mathrm{Poly}(P),\mathrm{Poly}(Q))=\{ax+by\ |\ x\in \mathrm{Poly}(P), y\in \mathrm{Poly}(Q), a+b=1, a,b\geq 0\}.$$
\end{definition} 
\begin{lemma}\label{L:C}Given two posets $P,Q$, the action of $\pcauset[alt={2-\hbox{chain} }]{1,2}$ on order polytopes satisfies:
    $$\mathrm{Poly}(\pcauset[alt={2-\hbox{chain} }]{1,2}(P,Q))=C(i(\mathrm{Poly}(P)),\mathrm{Poly}(Q)),$$ where 
    $i(\mathrm{Poly}(P))=\mathrm{Poly}(P)\times \overrightarrow{1}^{|Q|}$.
\end{lemma}
\begin{proof}
Consider the convex envelope of $\overrightarrow{0}^{|P|}\times \text{Poly}(Q)$ and $\text{Poly}(P)\times \overrightarrow{1}^{|Q|}$. 
We claim that every point in the line $\{\alpha(p,1)+(1-\alpha)(0,q)\ |\ 0\leq \alpha\leq 1\}$ satisfies the inequalities of the poset and that any of the first coordinates is smaller than any of the second coordinates.

For any such $\alpha$, the coordinates of $(1-\alpha) q$ and $\alpha p$ satisfy the inequalities of $\text{Poly}(Q)$ and $\text{Poly}(P)$, respectively, because we multiply every entry by a constant. Similarly, $\alpha+(1-\alpha) q$ preserves the inequalities entry-wise after translation and dilation.

Now, we have $\alpha p_i\leq \alpha\leq  \alpha+(1-\alpha)q_j$ for all $i,j$ corresponding indices. This means that the convex envelope of
$\overrightarrow{0}^{|P|}\times \text{Poly}(Q)$ and $\text{Poly}(P)\times \overrightarrow{1}^{|Q|}$ is the order polytope of $\pcauset[alt={2-\hbox{chain} }]{1,2}(P,Q)$.

    \end{proof}
    
\begin{definition}
An \textnormal{operad} is a sequence of sets $\{O(n)\}_{n\in\mathbb{N}}$, conceptualized as a set of $n$ -ary operations. It is equipped with composition morphisms $O(n)\times O(k_1)\times\cdots\times O(k_n)\rightarrow O(\sum_{i=1}^n k_i)$, which are associative.  There is an element in $O(1)$ which behaves as the unit.   
\end{definition}
A first introduction to operad theory can be found in the work of \cite{what} and \cite{whatis}. For a more detailed introduction to operads, we recommend~\cite{entropy}.

\begin{definition}
    The \textnormal{operad of finite posets} has as $n$-ary operations the posets with $n$ points, with composition given by the lexicogaphic sum. The one-point poset \pcauset[alt={1-\hbox{chain}}]{1} acts as an identity operation. 
\end{definition}
Note that $\pcauset[alt={x,y}]{2, 1}$ represents an associative and commutative operation, since the union of posets satisfies $P\sqcup Q=Q\sqcup P$.

Given a set $X$, the operad $End_X$ has as $n$-ary operations the functions $f:X^n\rightarrow X$, and the operadic composition is the composition of functions.
\begin{definition}
    An \textnormal{algebra} $A$ over the operad of posets $O$ is an operadic morphism $O\rightarrow End_A$.
\end{definition}

In other words, if $A$ is an algebra over an operad ${O}$, we realize the abstract $n$-ary operations $\mu\in{O}(n)$ as functions  $\mu:A^n\mapsto A$.

\begin{remark} Posets are an algebra over the operad of posets, where the action is the lexicographic sum of posets.
    \end{remark}

\begin{definition}
By $\chain[n]\circ_i \chain[m]$, we mean the action of $\chain[n]$ on the input $$(\chain[1],\dots,\chain[1],\chain[m],\chain[1],\dots,\chain[1])$$ where the entry $\chain[m]$ is located in the $i$-th position.    
\end{definition}

The language of category theory (operads) provides a framework that enables the application of techniques from order theory in other fields. We will now define the structure of algebra over the operad of posets on order polynomials, tropical polynomials, and a family of neural networks.

\begin{definition} 
Given a finite poset $P$ with $|P|=n$, and input posets $Q_1,\dots,Q_n$, we define the \textnormal{action of} $P$ \textnormal{on the order polytopes} of the input posets by:

$$P(\mathrm{Poly}(Q_1),\dots,\mathrm{Poly}(Q_n)):=\mathrm{Poly}(P(Q_1,\dots,Q_n)).$$
    
\end{definition}
Corollary~\ref{Cor:polytope} implies that the structure of an algebra over the operad of posets of order polytopes is well-defined because every order polytope is associated to a unique poset.
This way, we see posets parametrizing operations on order polytopes.

\begin{remark}   
We have seen that $\pcauset[alt={x,y}]{2, 1}$ acts as the Minkowski sum of order polytopes (see Lemma~\ref{L:M}), and $\pcauset[alt={2-\hbox{chain} }]{1, 2}$ is isomorphic to the convex envelope of order polytopes (see Lemma~\ref{L:C}).

The simplest non-series parallel poset is $P=\pcauset[alt={x<y>w<z}]{2, 4, 1, 3}$. Thus, the action of $\pcauset[alt={x<y>w<z}]{2, 4, 1, 3}$ is a generalization of the Minkowski sum and the convex envelope. The Minkowski sum is symmetric on its inputs, but for order polytopes $ C_1,C_2,C_3$ and $C_4$,  $\pcauset[alt={x<y>w<z}]{2, 4, 1, 3}(C_1,C_2,C_3,C_4)$ is in general not isomorphic to $\pcauset[alt={x<y>w<z}]{2, 4, 1, 3}(C_2,C_1,C_3,C_4)$. 

\end{remark}

\subsection{Tropical polynomials/neural networks of posets}
\label{end:nn}
In this section, when referring to the tropical polynomial of a poset $P$ we mean its expanded presentation (see Definition~\ref{def:exp}). 
We can associate to every order tropical polynomial a polytope, and we show that this association lifts the action of the operad of posets from order polynomials to tropical polynomials.

The following definition is justified by Lemma~\ref{Lemma:dist}.

\begin{definition}
Let $P$ be a poset with $n$ points and let $Q_1,\dots,Q_n$ be $n$ posets. We define the \textnormal{action of the poset} $P$ \textnormal{on the tropical polynomials} $\mathrm{Tr}(Q_1),\dots,\mathrm{Tr}(Q_n)$  by:  $$P(\mathrm{Tr}(Q_1),\dots,\mathrm{Tr}(Q_n)):=\mathrm{Tr}(P(Q_1,\dots,Q_n)).$$
    
\end{definition}

\begin{remark}
Note that the action of a chain $\chain[n]$ on the tropical polynomial of a chain $\chain[m]$ can be computed by evaluating $\mathrm{Tr}(\chain[n])$ on the polynomial $\mathrm{Tr}(\chain[m])$, as follows: 
    \begin{eqnarray}
        \chain[n](\mathrm{Tr}(1),\dots,\mathrm{Tr}(1),\mathrm{Tr}(m),\mathrm{Tr}(1),\dots,\mathrm{Tr}(1))&=&\mathrm{Tr}(\chain[n]_{\circ_i}\chain[m])\nonumber\\
        &=&\mathrm{Tr}(n+m-1)\nonumber\\
  &=&0\oplus x_{n+m-1} \otimes \Big( \cdots\otimes (0\oplus x_1)\cdots \Big)\nonumber
    \end{eqnarray}
    
    We used Equation~\eqref{Eqn:simplex} to go from the second to the third line. By expanding, the final value can be written as  $$0\oplus x_{n+m-1} \otimes ( \cdots\otimes (0\oplus x_{i+m}\otimes 
    \Big(0\oplus x_{i+m-1} \otimes ( \cdots\otimes (0\oplus x_i)\cdots \Big) 
    \otimes (0\oplus x_{i-1}\otimes(\cdots 0\oplus x_1)\cdots),$$ 
    which coincides with
    $$\mathrm{Tr}(n)(x_1,\dots,x_{i-1},\mathrm{Tr}(m)(x_i,\dots,x_{i+m-1}),x_{m+i+1},\dots,x_{m+n-1}).$$

\end{remark}

\begin{lemma}\label{Lemma:poly}
    The action of $\chain[n]$ on tropical polynomials of posets is given by evaluation on $\mathrm{Tr}(n)$ and relabeling of some variables.
\end{lemma}
\begin{proof}Let $P$ be a poset with $k$ points. The polynomial $\text{Tr}(\chain[n]\circ_i P)$ is the sum of tropical polynomials associated to maximal simplices in $\text{Poly}(\chain[n]\circ_i P)$. The maximal simplices of an order polytope are indexed by linearizations of the poset $\chain[n]\circ_i P$, as seen in~\cite{TwoPP}. 
Linearizations of $\chain[n]\circ_i P$ are linearizations of $P$ with a tail (coming from $\chain[i-1]\subset \chain [n]$) and a head (coming from $\chain[n-(i+1)]\subset \chain[n]$), and the tropical polynomials of these linearizations have the same coefficients as the evaluation $\text{Tr}(n)(x_1,\dots,x_{i-1},\text{Tr}( P),x_{k+i+1},\dots,x_{k+n-1})$, where we relabel the variables of the polynomial $\text{Tr}( P)$.
\end{proof}

Let $e(P)$ be the number of linearizations of the poset $P$.

\begin{lemma}
    Given a finite poset $P$ with $|P|=n$, and $Q_1,\dots, Q_n$ finite posets, we have that $\prod e(Q_i)$ divides $e(P(Q_1,\dots,Q_n))$.  
\end{lemma}\begin{proof}
    See~Lemma 2.3 of \cite{gold}. 
\end{proof}

The number of linearizations of the lexicographic sum
is a multiple of the number of linearizations of the input. Since $e(P(Q_1,\dots,Q_n))$ and $\prod e(Q_i)$  are in general not equal, which is required in the proof of Lemma~\ref{Lemma:poly},  we cannot expect that the action of a poset on tropical numbers will be given by the evaluation of the tropical polynomial of $P$ on the tropical polynomials of the inputs.

\begin{example}\label{ex:exp}
    For instance, compare the following evaluation of polynomials:
\begin{eqnarray*}
 \mathrm{Tr}(\pcauset[alt={x,y}]{2, 1})(\mathrm{Tr}(\pcauset[alt={2-\hbox{chain} }]{1, 2}),\mathrm{Tr}(\pcauset[alt={1-\hbox{chain} }]{1}))&=&0\oplus x(0\oplus y(0\oplus z))\\ 
 &\oplus& 0\oplus z(0\oplus x(0\oplus y)),
\end{eqnarray*}
with the action of the poset \pcauset[alt={x,y}]{2, 1}:
\begin{eqnarray*}
\pcauset[alt={x,y}]{2, 1}(\mathrm{Tr}(\pcauset[alt={2-\hbox{chain} }]{1, 2}),\mathrm{Tr}(\pcauset[alt={1-\hbox{chain} }]{1}))&=&\mathrm{Tr}(\pcauset[alt={x,y}]{2, 1}(\pcauset[alt={2-\hbox{chain} }]{1, 2},\pcauset[alt={1-\hbox{chain} }]{1}))\\
&=&0\oplus x(0\oplus y(0\oplus z))\\ 
 &\oplus&0\oplus x(0\oplus z(0\oplus y))\\ 
 &\oplus&0\oplus z(0\oplus x(0\oplus y)).
\end{eqnarray*}

\end{example}

In Definition~\ref{Def:NN}, we fixed a representative neural network for each poset.
\label{GeneralNN}

The following action is well defined up to the order of the linearizations on the inception module.

\begin{definition}
We define the \textnormal{action of the operad of posets on poset neural networks} by the rule:
$$P(nn(Q_1),\dots,nn(Q_n))=nn(P(Q_1,\dots,Q_n)).$$    \label{mainDef}
\end{definition}

\begin{example} From Example~\ref{ex:exp}, Table~\ref{fig:algebra1} and  Table~\ref{fig:algebra2}, we obtain
        \begin{eqnarray*}
\pcauset[alt={x,y}]{2, 1}(nn(\pcauset[alt={2-\hbox{chain} }]{1, 2}),nn(\pcauset[alt={1-\hbox{chain} }]{1}))&=&nn(\pcauset[alt={x,y}]{2, 1}(\pcauset[alt={2-\hbox{chain} }]{1, 2},\pcauset[alt={1-\hbox{chain} }]{1}))
\end{eqnarray*}
Explicitly,
$$\pcauset[alt={x,y}]{2, 1}\left({\mathrm{ReLU}_0\left(\begin{bmatrix}
\mathrm{ReLU}_0(x)&
\mathrm{ReLU}_{-\infty}(y)
\end{bmatrix}\begin{bmatrix}
1\\
1\\
\end{bmatrix}\right)},\mathrm{ReLU}_{0}(z)\right)$$
$$={\mathrm{ReLU}_0\left(\mathrm{ReLU}_{(0,-\infty)}\left(\begin{bmatrix}
\mathrm{ReLU}_0(x)&\mathrm{ReLU}_{-\infty}(y)&\mathrm{ReLU}_{-\infty}(z)\\
\mathrm{ReLU}_0(x)&\mathrm{ReLU}_{-\infty}(z)&\mathrm{ReLU}_{-\infty}(y)\\
\mathrm{ReLU}_0(z)&\mathrm{ReLU}_{-\infty}(x)&\mathrm{ReLU}_{-\infty}(y)\\
\end{bmatrix} 
\begin{bmatrix}
1&0\\
1&0\\
0&1
\end{bmatrix}\right)
\begin{bmatrix}
1\\
1\\
\end{bmatrix}\right)
}.$$
    \end{example}

\subsection{Extension to all tropical polynomials}\label{def:troppoly}
We have defined the endomorphisms of a family of neural networks. In this section, we will now extend the action of the operad of posets to convex polytopes. From the extension to convex polytopes, we will define the action of the operad of posets on tropical polynomials. 

 We define the action of the operad of posets on arbitrary points and arbitrary convex polytopes as follows. 
\begin{definition}\label{D:1}
Given $x_1,\dots, x_n\in \mathbb{R}^n$ and $P$, a poset with $n$ points, we define: $$P(x_1,\dots, x_n):=\{(a_1x_1,\dots, a_nx_n)\ |\ (a_1,\dots,a_n)\in \mathrm{Poly}(P)\}.$$
\end{definition}

\begin{definition}\label{D:2}
    
Let $P$ be a poset with $n$ points and let $C_1,\dots,C_n$ be convex polytopes. Then, we define the \textnormal{action of $P$ on convex polytopes} as: $$P(C_1,\dots,C_n):=\sqcup_{\{c_1\in C_1,\dots, c_n\in C_n\}} P(c_1,\dots,c_n).$$
\end{definition}

In particular, if all polytopes $\{C_i\}_{1\leq i\leq n}$ have integer vertices, then the resultant polytope $P(C_1,\dots,C_n)$ has integer vertices. 

Note that, if $P$ is a poset with $|P|=n$ and $I^i=\overrightarrow{0}^{i-1}\times[0,1]\times\overrightarrow{0}^{n-i}$, then $P(I^1,\dots,I^n)=P(\text{Poly}(1),\dots,\text{Poly}(1))=
\text{Poly}(P(1,\dots,1))=\text{Poly}(P)$.

We now show that the action of a poset on convex sets is a convex set.
\begin{lemma}\label{L:Convp}
    Let $P$ be a poset with $n$ points and let $C_1,\dots,C_n$ be convex polytopes. Let the polytope $P(C_1,\dots, C_n)$ have vertices $v_1,\dots, v_r$. Then $C(\{v_1,\dots, v_r\})=P(C_1,\dots, C_n)$, where the polytope $P(C_1,\dots, C_n)$ is given by the action of the operad of posets, and $C$ is the convex envelope.
\end{lemma}

\begin{proof}
The polytope $\mathrm{Poly}(P)$ is convex because if the coordinates of $x,y$ satisfy the inequalities of the poset $P$, then the coordinates of $tx+(1-t)y$ satisfy the inequalities as well.

 Given $x, y\in P(C_1,\dots, C_n)$,
 assume $x\in P(r_1,\dots,r_n) $ and $y\in P(s_1,\dots,s_n)$. Then, the line $tx+(1-t)y\in P(t((r_1,\dots,r_n))+(1-t)(s_1,\dots,s_n))\subset P(C_1,\dots, C_n)$ since each $tr_i+(1-t)s_i\in C_i$.

Therefore, the action of posets on convex polytopes preserves convexity.
\end{proof}

\begin{remark}
Note that if the inputs are not convex, then the polytope may not be convex, for example, $\pcauset[alt={2-\hbox{chain} }]{1, 2}\Big(\{(1,0,0),(0,1,0)\}, \{(0,0,1)\}\Big)$ gives two lines with the same end point, but there is no line between $(1,0,0)$ and $(0,1,0)$ in the set.
\end{remark}

\begin{remark}
 Note also that a tropical polynomial is associated with a polytope where all but the last coordinate are guaranteed to be non-negative integers. 
 \end{remark}

We can now define the action of posets on tropical polynomials.

\begin{definition} Let $P$ be a poset with $n$ points, and let $f_1, \dots, f_n$ be $n$ tropical polynomials with their corresponding convex polytopes $C_1, \dots, C_n$.
As polytopes, we compute the action $P(C_1, \dots, C_n)$, obtain the vertices of $P(C_1, \dots, C_n)$, and then these vertices define a tropical polynomial, which we denote by $P(f_1,\dots, f_n)$. \label{TPA}
\end{definition}
This composition sends $n$ tropical polynomials into a tropical polynomial.
\begin{example}\label{ex:ex}
    The order polytope of $\pcauset[alt={x<z, y<z}]{2, 1, 3}=\{\{x,y\}\leq z\}$ is an upside-down pyramid with a quadrangular base.
    Then, the associated order polytope of $\pcauset[alt={x<z, y<z}]{2, 1, 3}(x,x^2,z)$ is the degenerated pyramid obtained after transforming the canonical basis into $\{(1,0,0),(2,0,0),(0,0,1)\}$. Thus, $\pcauset[alt={x<z, y<z}]{2, 1, 3}(x,x^2,z)=0\oplus z \oplus xz \oplus x^2z \oplus x^3z$.
\end{example}
\begin{example}
   Consider $\pcauset[alt={x<z, y<z}]{2, 1, 3}$ with the labels of its points as in Example~\ref{ex:ex}. Then, the polytope of $\pcauset[alt={x<z, y<z}]{2, 1, 3}(x,x^2\oplus y,z)$
   is the following union of the polytopes $$C\left((0,0,0),(0,0,1),(1,0,1),(2t,(1-t),1),(1+2t,(1-t),1)\right).$$
   As a homotopy between the upside-down pyramid and a triangle, the corners of the polytope are $(0,0,0),(0,0,1),(1,0,1),(2,0,1),(3,0,1),(0,1,1),(1,1,1)$. Thus,  
   $$\pcauset[alt={x<z, y<z}]{2, 1, 3}(x,x^2\oplus y,z)=0\oplus z\oplus xz\oplus x^2z\oplus x^3z\oplus yz\oplus xyz.$$
\end{example}

\subsection{Future work}
A major problem in the theory of algebras over the operad of posets is the precise description of the action of posets. The work of~\cite{SP} computes explicitly the action of the operad of series parallel posets on Stanley-order polynomials, while Section 2 of the paper by~\cite{OP} provide a description of the action of the poset \pcauset[alt={x<y>w<z}]{2, 4, 1, 3} on Stanley order polynomials without explicitly computing it.

As is common in category theory, understanding these operations should have implications in different fields. For example,~\cite{shuffle} provide the precise description of the operations \pcauset[alt={2-\hbox{chain} }]{1, 2} and \pcauset[alt={x,y}]{2, 1} acting on ``shuffles of posets'',  to answer a question with roots in dendroidal homotopy theory.

Although the Minkowski sum and the convex envelope are well-known operations, we are not aware of the study of other poset endomorphisms in geometry.

We restrict our attention to tropical polynomials in order to provide a geometric interpretation of our results. But it should be possible to understand the geometric information coming from tropical rational functions. For example, the paper by~\cite{TG}, in particular Proposition 6.1, describes the decision region of a neural network associated to a tropical rational function. We wonder if the theory of virtual polytopes described in~\cite{VP} can be related to tropical quotients of posets neural networks.

In this paper, we were able to extend endomorphisms of posets to endomorphisms of poset neural networks, but the general extension to IVNN is still open.

% Acknowledgements and Disclosure of Funding should go at the end, before appendices and references

{This project started after Iryna Raievska and Maryna Raievska asked the first author about the applications of (operadic) order theory to machine learning during the Ukraine Algebra Conference ``At the End of the Year 2023'', and we thank them for their question.

The work of S.I. Kim, E. Dolores-Cuenca and S. L\'opez-Moreno was supported by the National Research Foundation of Korea (NRF) grant funded by the Korean Government (MSIP) (2022R1A5A1033624, 2021R1A2B5B03087097) and Global—Learning and Academic research institution for Master’s·PhD students, and Postdocs (LAMP) Program of the National Research Foundation of Korea (NRF) grant funded by the Ministry of Education (No. RS-2023-00301938).
S. L\'opez-Moreno was also supported by the Korea National Research Foundation (NRF) grant funded by the Korean government (MSIT) (RS-2024-00406152). J.L. Mendoza-Cortes acknowledges startup funds from Michigan State University. This work was supported in part through computational resources and services provided by the Institute for Cyber-Enabled Research at Michigan State University. }


\newpage

\appendix
\section{Additional information on experiments}
\label{app:theorem}

% Note: in this sample, the section number is hard-coded in. Following
% proper LaTeX conventions, it should properly be coded as a reference:

%In this appendix we prove the following theorem from
%Section~\ref{sec:textree-generalization}:

\subsection{Datasets used for experiments}\label{Sec:dataset}
Our experiments used the CIFAR10 dataset and the CIFAR100 dataset, described by~\cite{CIFARS}; the Fashion MNIST dataset, introduced by~\cite{FMNIST}; and a subset of ImageNet, created by~\cite{imagenet}. The subset, commonly refered as ImageNet100, was selected according to~\cite{imagenet100pytorch}.

\subsection{Choices made during the experiments}\label{Sec:choices}
Since the disjoint union of points has too many parameters and the four chain has too little, most of our experiments use the $\pcauset[alt={x<y>w<z}]{2, 4, 1, 3}$ filter.

 All experiments, including the implementations of the original architectures, used seeds to guarantee reproducibility, except those reported in the literature or with an explicit note. 
The average accuracy of the experiments without the reproducibility seeds was higher than the average accuracy after adding the reproducibility seeds. This behavior is well documented in the official documentation of PyTorch \url{https://pytorch.org/docs/stable/notes/randomness.html}.

Unless explicitly remarked, we used the default hyperparameters on each architecture, meaning, the default values on the Github repositories.

\section{Detailed experimental results}
Intuition tells us that our poset filter works best as a pooling function that is introduced almost at the end of the convolutional part of a CNN, but not immediately before the dense network. We believe that introducing the poset filter before one of the last convolutions still provides a significant reduction of parameters, and experimental results suggest that introducing it at this point results in minimal loss of information compared to a max pooling, average pooling, and mixed pooling.

\subsection{Quaternion Convolutional Neural Networks}
\label{Sec:exp}

Quaternion numbers can encode rotations with the advantage of using fewer parameters, since multiplication with a quaternion with four coordinates encodes the same information as multiplication with a ($3\times 3$) matrix with $9$ parameters.  
A quaternion neural network was studied by~\cite{QRNN}, with the motivation that quaternion neural networks require fewer parameters than real neural networks to achieve similar precision.

In our experiments our aim is to test a new pooling filter that will effectively reduce the dimensions of the input by half. We choose to work with quaternion neural networks as the quaternion assumption implies that all layers of the quaternion have inputs with dimensions divisible by 2. 

Our first experiment consists of inserting our function into their shallow neural network applied on CIFAR10. We compare the results with not only the ones obtained by~\cite{oQCNN}, but also with an alternative implementation in Pytorch developed by J. Weeda, R. Awad and M. Msallak in\\ \url{https://github.com/JorisWeeda/Quaternion-Convolutional-Neural-Networks/tree/main} (called QCNN (Weeda impl.) in the results' tables). This shallow network has two blocks of double convolutions followed by their respective ReLU activation functions, and we show the performance results of the model depending on the position of our function: after the first convolution (and its corresponding ReLU), before the last convolution and, lastly, immediately before the dense part. We also performed a $k$-fold cross-validation to demonstrate the robustness and generalizability of our model.

We have run our model on a V100 GPU and implemented it in Pytorch 2.4.1, as described by ~\cite{PyTorch}, with support for CUDA 12.4, following~\cite{CUDA}. Although we have kept the architecture mainly unchanged to facilitate a more direct comparison of the results, we have included reproducibility seeds, and, therefore, we will also include the results of the QCNN (Weeda impl.) when adding those same reproducibility seeds.

\subsubsection{Results when $\pcauset[alt={x<y>w<z}]{2, 4, 1, 3}$ filter is inserted after the first convolution}
%\\
In Table~\ref{table:exp11} we present the accuracy of our model when the $\pcauset[alt={x<y>w<z}]{2, 4, 1, 3}$ filter is inserted after the first convolution, calculated as the average accuracy over $14$ training runs. The results also include the accuracy when substituting the inserted $\pcauset[alt={x<y>w<z}]{2, 4, 1, 3}$ filter by a max pooling, average pooling, and mixed pooling.

In Table~\ref{table:exp13} we include the results when performing a $k$-fold cross-validation with $k=5$. The sample standard deviation (std) of the results was also calculated.

\begin{table}[h]
\centering
\begin{tabular}{|c||c|c|c|c|c|}
\hline
QCNN& Training& \# params & Time & Test acc & Std\\ 
&runs& & & &\\\hline\hline
Original & 1 & 2,032,650 & -- & 77.78\% & -- \\ \hline
Weeda impl. & 3 & 2,032,650 & 39' 34'' & \textbf{78.14}\% & -- \\
without seeds &&&&&\\ 
\hline
Weeda impl. & 14 & 2,032,650 & 25' 40'' & 77.83\% & 0.730\\ \hline
Weeda impl. + \pcauset[alt={x<y>w<z}]{2, 4, 1, 3} & 14 & 459,786 & 33' 15'' & 72.38\% & 0.835 \\ \hline
Weeda impl. & 14 & 459,786 & 33' 40'' & 74.64\% & 0.602 \\ 
+ max pooling &&&&&\\
(instead of \pcauset[alt={x<y>w<z}]{2, 4, 1, 3}) &&&&&\\ \hline
Weeda impl.& 14 & 459,786 & 28' 15'' & 70.15\% & 0.798 \\ 
 + avg. pooling  &&&&&\\
(instead of \pcauset[alt={x<y>w<z}]{2, 4, 1, 3}) &&&&&\\ \hline
Weeda impl.& 14 & 459,786 & 32' 40'' & 74.08\% & 0.816 \\ 
 + mixed pooling  &&&&&\\
(instead of \pcauset[alt={x<y>w<z}]{2, 4, 1, 3}) &&&&&\\ \hline
% \hline
\end{tabular}
\caption{Experiment results on CIFAR10 classification task for quaternion convolutional neural networks with 80 epochs}
    \label{table:exp11}
\end{table}

\begin{table}[h]
\centering
\begin{tabular}{|c||c|c|c|c|}
\hline
QCNN& \# folds & \# params & Val acc & Std  \\ \hline\hline
Weeda impl. & 5 & 2,032,650 & \textbf{75.54}\% & 0.936 \\ 
%$+$ seeds &&&&\\ 
\hline
Weeda impl. + \pcauset[alt={x<y>w<z}]{2, 4, 1, 3} & 5 & 459,786 & 70.38\% & 1.014  \\ \hline
Weeda impl. + max pooling& 5 & 459,786 & 72.74\% & 0.589 \\ 
(instead of \pcauset[alt={x<y>w<z}]{2, 4, 1, 3}) &&&&\\ \hline
Weeda impl. + avg. pooling & 5 & 459,786 & 68.09\% & 0.857 \\ 
(instead of \pcauset[alt={x<y>w<z}]{2, 4, 1, 3}) &&&&\\ \hline
Weeda impl. + mixed pooling & 5 & 459,786 & 71.73\% & 0.630 \\ 
(instead of \pcauset[alt={x<y>w<z}]{2, 4, 1, 3}) &&&&\\ \hline
\end{tabular}
\caption{Cross-validation results on CIFAR10 classification task for quaternion convolutional neural networks with 80 epoch}
    \label{table:exp13}
\end{table}

\subsubsection{Results when $\pcauset[alt={x<y>w<z}]{2, 4, 1, 3}$ filter is inserted in the optimal position: inserted before the last convolution} 

In Table~\ref{table:exp14} we present the accuracy of our model when the $\pcauset[alt={x<y>w<z}]{2, 4, 1, 3}$ filter is inserted before the last convolution, calculated as the average accuracy over $14$ training runs with different seeds each. Table~\ref{table:exp16} shows the results for the $k$-fold cross-validation with $k=5$. \par 

\begin{table}[h]
\centering
\begin{tabular}{|c||c|c|c|c|c|}
\hline
QCNN & Training& \# params & Time & Test acc & Std \\
&runs&&&&\\\hline\hline
Original & 1 & 2,032,650 & -- & 77.78\% & --  \\ \hline
Weeda impl. & 3 & 2,032,650 & 39' 34''& 78.14\% & --\\ 
wihtout seeds&&&&&\\
\hline
Weeda impl. & 14 & 2,032,650 & 25' 40''& 77.83\% & 0.730 \\ 
%$+$ seeds &&&&&\\
\hline
Weeda impl.  & 14 & 656,394 & 34' 20'' & \textbf{78.92}\% & 0.514 \\ 
+ \pcauset[alt={x<y>w<z}]{2, 4, 1, 3}&&&&&\\
\hline
Weeda impl.  & 14 & 656,394 & 27' 15'' & 76.75\% & 0.397 \\ 
 + max pooling&&&&&\\
(instead of \pcauset[alt={x<y>w<z}]{2, 4, 1, 3}) &&&&&\\ \hline
Weeda impl.   & 14 & 656,394 & 28' 15'' & 76.097\% & 1.104 \\
 + avg. pooling &&&&&\\
 (instead of \pcauset[alt={x<y>w<z}]{2, 4, 1, 3}) &&&&&\\
\hline
Weeda impl.  & 14 & 656,394 & 27' 10'' & 76.41\% & 0.670 \\
 + mixed pooling   &&&&&\\
 (instead of \pcauset[alt={x<y>w<z}]{2, 4, 1, 3}) &&&&&\\
\hline
\end{tabular}
\caption{Experiment results on CIFAR10 classification task for quaternion convolutional neural networks with 80 epochs}
    \label{table:exp14}
\end{table}

\begin{table}[h]
\centering
\begin{tabular}{|c||c|c|c|c|}
\hline
QCNN& \# folds & \# params & Val acc & Std \\ \hline\hline
 Weeda impl. & 5 & 2,032,650 & 75.54\%  & 0.936  \\ 
%$+$ seeds &&&&\\
\hline
Weeda impl. + \pcauset[alt={x<y>w<z}]{2, 4, 1, 3} & 5 & 656,394 & \textbf{76.78}\% & 0.517   \\ \hline
Weeda impl. + max pooling& 5 & 656,394 &  74.23\% & 0.964   \\ 
(instead of \pcauset[alt={x<y>w<z}]{2, 4, 1, 3}) &&&&\\ \hline
Weeda impl. + avg. pooling & 5 & 656,394 & 73.01\% & 1.231  \\ (instead of \pcauset[alt={x<y>w<z}]{2, 4, 1, 3}) &&&&\\ \hline
Weeda impl. + mixed pooling & 5 & 656,394 & 73.23\% & 1.035  \\(instead of \pcauset[alt={x<y>w<z}]{2, 4, 1, 3}) &&&&\\  \hline
\end{tabular}
\caption{Cross-validation results on CIFAR10 classification task for quaternion convolutional neural networks with 80 epochs}
    \label{table:exp16}
\end{table}

\subsubsection{Results when $\pcauset[alt={x<y>w<z}]{2, 4, 1, 3}$ filter is inserted immediately before the dense network}

In Table~\ref{table:exp17} we present the accuracy of our model when the $\pcauset[alt={x<y>w<z}]{2, 4, 1, 3}$ filter is inserted immediately before the dense part, calculated as the average accuracy over $14$ training runs. In Table~\ref{table:exp19} appear the results for the $k$-fold cross-validation with $k=5$.\par

\begin{table}[h]
\centering
\begin{tabular}{|c||c|c|c|c|c|}
\hline
QCNN & Training& \# params & Time & Test acc & Std\\ 
&runs&&&&\\\hline\hline
Original & 1 & 2,032,650 & -- & 77.78\% & --  \\ \hline
Weeda impl. & 3 & 2,032,650 & 39' 34'' & 78.14\%  & -- \\
without seeds &&&&&\\ 
\hline
Weeda impl. & 14 & 2,032,650 & 25' 40'' & 77.83\% & 0.730 \\ \hline
Weeda impl.   & 14 & 656,394 & 32' 40'' & \textbf{78.11}\% & 0.876 \\ 
+\pcauset[alt={x<y>w<z}]{2, 4, 1, 3}&&&&&\\
\hline
Weeda impl.  & 14 & 656,394 & 26' 15'' & 75.05\% & 1.169 \\ 
 + max pooling&&&&&\\
(instead of \pcauset[alt={x<y>w<z}]{2, 4, 1, 3}) &&&&&\\ \hline
Weeda impl.  & 14 & 656,394 & 28' 10'' & 75.8\% & 0.835 \\
 + avg. pooling &&&&&\\
 (instead of \pcauset[alt={x<y>w<z}]{2, 4, 1, 3}) &&&&&\\\hline
Weeda impl.  & 14 & 656,394 & 30' 15'' & 74.64\% & 1.354 \\
+ mixed pooling &&&&&\\
(instead of \pcauset[alt={x<y>w<z}]{2, 4, 1, 3}) &&&&&\\\hline
\end{tabular}
\caption{Experiment results on CIFAR10 classification task for quaternion convolutional neural networks with 80 epochs}
    \label{table:exp17}
\end{table}

\begin{table}[h]
\centering
\begin{tabular}{|c||c|c|c|c|}
\hline
QCNN& \# folds & \# params & Val acc & Std\\ \hline\hline
 Weeda impl. & 5 & 2,032,650 & 75.54 & 0.936  \\ 
%$+$ seeds &&&&\\
\hline
Weeda impl. + \pcauset[alt={x<y>w<z}]{2, 4, 1, 3} & 5 & 656,394 & \textbf{76.21}\% & 0.929   \\ \hline
Weeda impl. + max pooling& 5 & 656,394 & 71.85\% & 0.793 \\ 
(instead of \pcauset[alt={x<y>w<z}]{2, 4, 1, 3}) &&&&\\ \hline
Weeda impl. + avg. pooling & 5 & 656,394 & 72.88\% & 0.758 \\ (instead of \pcauset[alt={x<y>w<z}]{2, 4, 1, 3}) &&&&\\ \hline
Weeda impl. + mixed pooling & 5 & 656,394 & 71.7\% & 1.981 \\ (instead of \pcauset[alt={x<y>w<z}]{2, 4, 1, 3}) &&&&\\ \hline
\end{tabular}
\caption{Cross-validation results on CIFAR10 classification task for quaternion convolutional neural networks with 80 epochs}
    \label{table:exp19}
\end{table}

\subsubsection{Additional experiments}

So far, we have conducted experiments using the $\pcauset[alt={x<y>w<z}]{2, 4, 1, 3}$ poset filter. However, there are in total $16$ different posets with $4$ vertices or nodes, including the $\pcauset[alt={x<y>w<z}]{2, 4, 1, 3}$ poset. Therefore, we conduct experiments on the remaining $15$ posets to compare their performance as well. See results in Table~\ref{table:exp110}. 

Moreover, since the $\pcauset[alt={x<y>w<z}]{2, 4, 1, 3}$ poset is defined as the maximum of $8$ specific terms, one of which is $0$, we conduct experiments with $30$ random functions that take the maximum between $0$ and $7$ random linear combinations of the four inputs. The corresponding results appear in Table~\ref{table:exp111}. 

\begin{table}[h]
\centering
\begin{tabular}{|c||c|c|c|c|c|c|}
\hline
QCNN& Training runs & \# params & Time & Test acc & Std\\ 
Weeda impl. + &&&&&\\\hline\hline

 \pcauset[alt={x<z,y<z,w<z}]{3,2,1,4} & 14 & 656,394 & 43' 40'' & 78.78\% & 0.646\\ \hline
 \pcauset[alt={x,y,w,z}]{4,3,2,1} & 14 & 656,394 & 40' 55'' & 78.95\% & 0.510\\ \hline
 \pcauset[alt={x,y,w<z}]{4,3,1,2} & 14 & 656,394 & 39' 50'' & 78.79\% & 0.609\\ \hline
 \pcauset[alt={x<y<z,w<y}]{2,1,3,4} & 14 & 656,394 & 38' 10'' & 78.94\% & 0.429\\ \hline
 \pcauset[alt={x,y<z,w<z}]{4,2,1,3} & 14 & 656,394 & 37' 50'' & 78.91\% & 0.645\\ \hline
 \pcauset[alt={x,y<w,y<z}]{4,1,3,2} & 14 & 656,394 & 37' 50'' & 78.92\% & 0.595\\ \hline
 \pcauset[alt={x<y,w<z}]{3,4,1,2} & 14 & 656,394 & 37' 15'' & 78.74\% & 0.566\\ \hline
 \pcauset[alt={x<w<z, w<y }]{1, 2, 4, 3} & 14 & 656,394 & 36' 20'' & \textbf{79.1}\% & 0.449 \\ \hline
 \pcauset[alt={x,y<w<z}]{4,1,2,3} & 14 & 656,394 & 35' 30'' & 79.01\% & 0.706\\ \hline
 \pcauset[alt={x<y,x<w,x<z}]{1,4,3,2} & 14 & 656,394 & 35' 25'' & 79\% & 0.559\\ \hline
 \pcauset[alt={x<y<z,x<w}]{1,3,4,2} & 14 & 656,394 & 34' 40'' & 78.98\% & 0.507\\ \hline
 \pcauset[alt={x<y,x<z,w<y,w<z}]{2,1,4,3} & 14 & 656,394 & 34' 25'' & 78.95\% & 0.742\\ \hline
 \pcauset[alt={x<y>w<z}]{2, 4, 1, 3} & 14 & 656,394 & 34' 20'' & 78.92\% & 0.514\\ \hline
 \pcauset[alt={x,y,w,z}]{1,2,3,4} & 14 & 656,394 & 33' 30'' & 78.92\% & 0.616 \\ \hline
 \pcauset[alt={x<y<z,x<w<z}]{1,3,2,4} & 14 & 656,394 & 33' 25'' & 78.96\% & 0.442 \\ \hline
 \pcauset[alt={x<y<z,w<z}]{3,1,2,4} & 14 & 656,394 & 33' 25'' & 78.81\% & 0.678 \\ \hline
\end{tabular}
\caption{Experiment results for all $4$-vertex-posets on CIFAR10 classification task for quaternion convolutional neural networks with 80 epochs}
    \label{table:exp110}
\end{table}

To understand how the choice of poset affects the encoding properties of a poset filter, we sample points in a lattice of points of the form $(\pm \frac{a}{25}, \ \pm \frac{b}{25}, \ \pm \frac{c}{25},\ \pm \frac{d}{25})$, with $0\leq  a,b,c,d\leq 25$, and we consider only those points inside the unit four-dimensional ball. Figure~\ref{fig:enter-Histogram} shows the histogram of positive values and their corresponding standard deviation. The histogram was created with Matplotlib (\cite{MPL}). Although most values are close to zero, there is a large difference among the standard deviations.
\begin{figure}[h]
    \centering
\includegraphics[width=1\linewidth]{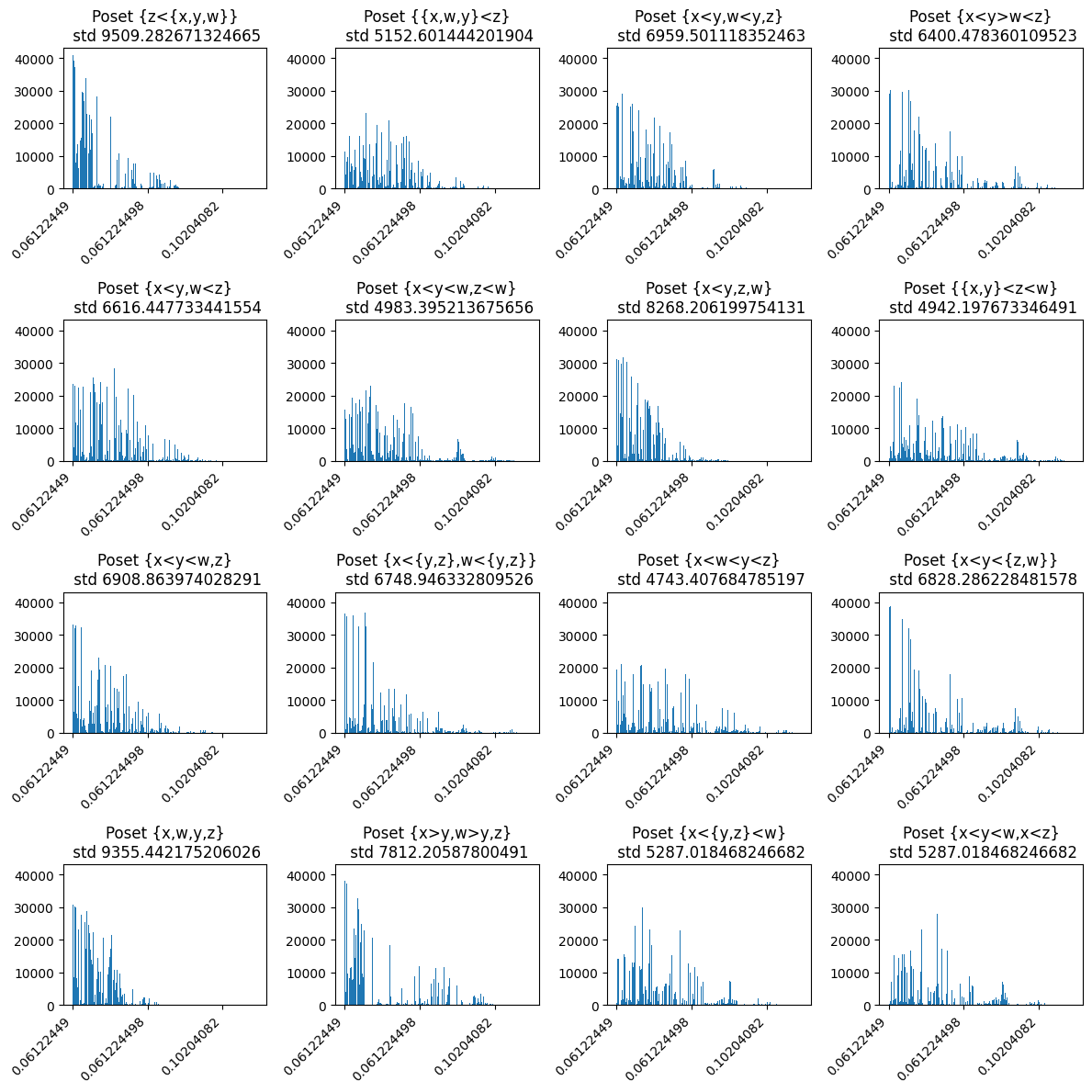}
    \caption{Histogram of the image of the lattice points in the sphere under different posets transformations.}
    \label{fig:enter-Histogram}
\end{figure}

\begin{table}[h]
\centering
\begin{tabular}{|c||c|c|c|c|c|c|}
\hline
QCNN & Training& \# params & Time & Test acc & Std \\
Weeda impl. + &runs&&&&\\\hline\hline
 Rndm vt \# 1 & 10 & 656,394 & 50' 25'' & 78.11\% & 0.758\\ \hline
 Rndm vt \# 2 & 10 & 656,394 & 53' 35'' & 78.05\% & 0.736\\ \hline
 Rndm vt \# 3 & 10 & 656,394 & 47' 15'' & 78.03\% & 0.805\\ \hline
 Rndm vt \# 4 & 10 & 656,394 & 42' 10'' & 78.09\% & 0.705 \\ \hline
 Rndm vt \# 5 & 10 & 656,394 & 39' 55'' & 77.91\% & 0.642\\ \hline
 Rndm vt \# 6 & 10 & 656,394 & 38' 20'' & 78.22\% & 0.629\\ \hline
 Rndm vt \# 7 & 10 & 656,394 & 36' 15'' & 78.16\% & 0.569\\ \hline
 Rndm vt \# 8 & 10 & 656,394 & 35' 25'' & 78.09\% & 0.539\\ \hline
 Rndm vt \# 9 & 10 & 656,394 & 34' 55'' & 78.18\% & 0.769\\ \hline
 Rndm vt \# 10 & 10 & 656,394 & 34' 40'' & \textbf{78.38}\% & 0.734\\ \hline
 Rndm vt \# 11 & 10 & 656,394 & 32' 50'' & 77.73\% & 0.725\\ \hline
 Rndm vt \# 12 & 10 & 656,394 & 32' 25'' & 77.73\% & 0.746\\ \hline
\end{tabular}
\caption{Experiment results for random vectors on CIFAR10 classification task for quaternion convolutional neural networks with 80 epochs}
    \label{table:exp111}
\end{table}

\subsection{CNN}\label{Sec:ex3}

We work with Fashion MNIST with the architecture from \url{https://github.com/Abhi-H/CNN-with-Fashion-MNIST-dataset}, a standard convolutional neural network with ELU, a non linearity, defined by~\cite{ELU}, and two max pooling layers. The original network achieves $90.65\%$ accuracy. Our best result was a $91.84\%$ accuracy, with an average accuracy of $91.26\%$ over $14$ runs, obtained by replacing the second max pooling by the $\pcauset[alt={x<y>w<z}]{2, 4, 1, 3}$ filter. We also include the result of replacing the second max pooling by average pooling and mix pooling, see Table~\ref{table:expCNN}.

\begin{table}[h]
\centering
\begin{tabular}{|c||c|c|c|}
\hline
    CNN &  Training runs  & Test acc                & Std\\ 
    \hline\hline
      Original&14               & $ 90.65\%$              &0.36\\ 
      \hline
Replaced &14   &  $90.268\%$              & 0.234\\
first max pooling& &&\\
by $\pcauset[alt={x<y>w<z}]{2, 4, 1, 3}$ &&&\\
\hline
Replaced&14   &  $\textbf{91.26}\%$              & 0.466\\
second max pooling& &&\\
by $\pcauset[alt={x<y>w<z}]{2, 4, 1, 3}$ &&&\\
\hline
Replaced&14   &  $90.747\%$              &0.404 \\
both max pooling& &&\\
by $\pcauset[alt={x<y>w<z}]{2, 4, 1, 3}$ &&&\\
\hline
Replaced&14   &  ${90.41}\%$              & 0.454\\
second max pooling& &&\\
by mix pooling &&&\\
\hline
Replaced&14   &  ${90.15}\%$              & 0.313\\
second max pooling& &&\\
by avg. pooling &&&\\
\hline    
\end{tabular}
\caption{Experiment results on a real CNN for Fashion MNIST classification task, replacing max poolings in different positions by the $\pcauset[alt={x<y>w<z}]{2, 4, 1, 3}$ poset.\label{table:expCNN}}
\end{table}

In Table~\ref{table:exp2} we  report cross-validation with $k=5$. All experiments were performed on NVIDIA A100-SXM4-80GB.

\begin{table}[h]
\centering
\begin{tabular}{|c||c|c|c|c|}
\hline
CNN & Training runs & \# params & Val acc &Std\\ 
\hline\hline
{Original}& 5 &28938 &90.521 & 0.151\\
 
&&&&\\
\hline
Replaced & 5 & 28938& \textbf{90.97} &0.321 \\
 second max pooling &&&&\\
 by $\pcauset[alt={x<y>w<z}]{2, 4, 1, 3}$ &&&&\\
 \hline
\end{tabular}
\caption{Cross-validation results on Fashion MNIST classification task for a real convolutional neural network with 80 epochs}
    \label{table:exp2}
\end{table}

We also implemented \url{https://pytorch.org/tutorials/beginner/blitz/cifar10_tutorial.html} on CIFAR10, the only modification being the use of the correct parameters for the normalization transformation of the dataset. The test accuracy of the original architecture with our implementation (with seeds) has an average of $60.41\%$. Our best result was $60.71\%$, replacing the second max pooling by the $\pcauset[alt={x<y>w<z}]{2, 4, 1, 3}$ filter. We then replaced the second max pooling by avg pooling and obtained an average test accuracy of $62.13\%$. 
See Table~\ref{table:expCNNCw}.

%~\ref{table:expCNNCw1}.

\begin{table}[h]
\centering
\begin{tabular}{|c||c|c|c|}
\hline
    CNN &  Training runs  & Test acc                & Std\\ 
    \hline\hline
      Original&14               & $ 60.41\%$              &0.9\\ 
      \hline
Replaced &14   &  $57.288\%$              & 0.58\\
first max pooling& &&\\
by $\pcauset[alt={x<y>w<z}]{2, 4, 1, 3}$ &&&\\
\hline
Replaced&14   &  ${60.71}\%$              & 0.784\\
second max pooling& &&\\
by $\pcauset[alt={x<y>w<z}]{2, 4, 1, 3}$ &&&\\
\hline
Replaced&14   &  $57.43\%$              &0.74 \\
both max pooling& &&\\
by $\pcauset[alt={x<y>w<z}]{2, 4, 1, 3}$ &&&\\
\hline
Replaced&14   &  60.17\%              & .82\\
second max pooling& &&\\
by mix pooling &&&\\
\hline
Replaced&14   &  $\textbf{62.13}\%$              & .78\\
second max pooling& &&\\
by avg. pooling &&&\\
\hline    
\end{tabular}
\caption{Experiment results on the tutorial CNN for CIFAR 10 classification task, replacing max poolings in different positions by the $\pcauset[alt={x<y>w<z}]{2, 4, 1, 3}$ poset.\label{table:expCNNCw}}
\end{table}

\subsection{SimpleNet}\label{sec:simp}
In this experiment we implemented the official PyTorch repository \url{https://github.com/Coderx7/SimpleNet_Pytorch}, where we used the architecture `simplenetv1\_9m\_m2' with 9m parameters. According to the file \url{cifar/models/simplenet.py} in the official repository, the arquitecture contains two layers labeled with `p'. They consist on a MaxPool followed by a Dropout. We considered four architectures: the original, one in which we modify the first layer `p' by replacing the MaxPool with the $\pcauset[alt={x<y>w<z}]{2, 4, 1, 3}$  filter, a third architecture in which we replaced  the MaxPool by the $\pcauset[alt={x<y>w<z}]{2, 4, 1, 3}$ filter in the second `p' layer, and a final architecture in which we replaced the MaxPool by the $\pcauset[alt={x<y>w<z}]{2, 4, 1, 3}$ filter in both `p' layers. 

We evaluated on ImageNet100, and we repeated experiments three times. More repetitions are currently running.

We use the same one hundred classes of ImageNet100 as~\cite{imagenet100pytorch}. Each of the 100 validation classes contains 50 images. 

Experiments consisted of choosing an architecture (out of the four possibilities) and training  and validating on the public train/validation data for 700 epochs. We repeated the previous step six times (having in total 24 experiments). %Then looking at the plots we realized that the point with lowest validation loss occurs just before epoch 200. We then  repeated the experiments 6 more times (per architecture) but training only up to epoch 200.
%For each architecture, we have in total 3 repetitions.
\begin{figure}[h]
    \centering
    \includegraphics[width=.95\linewidth]{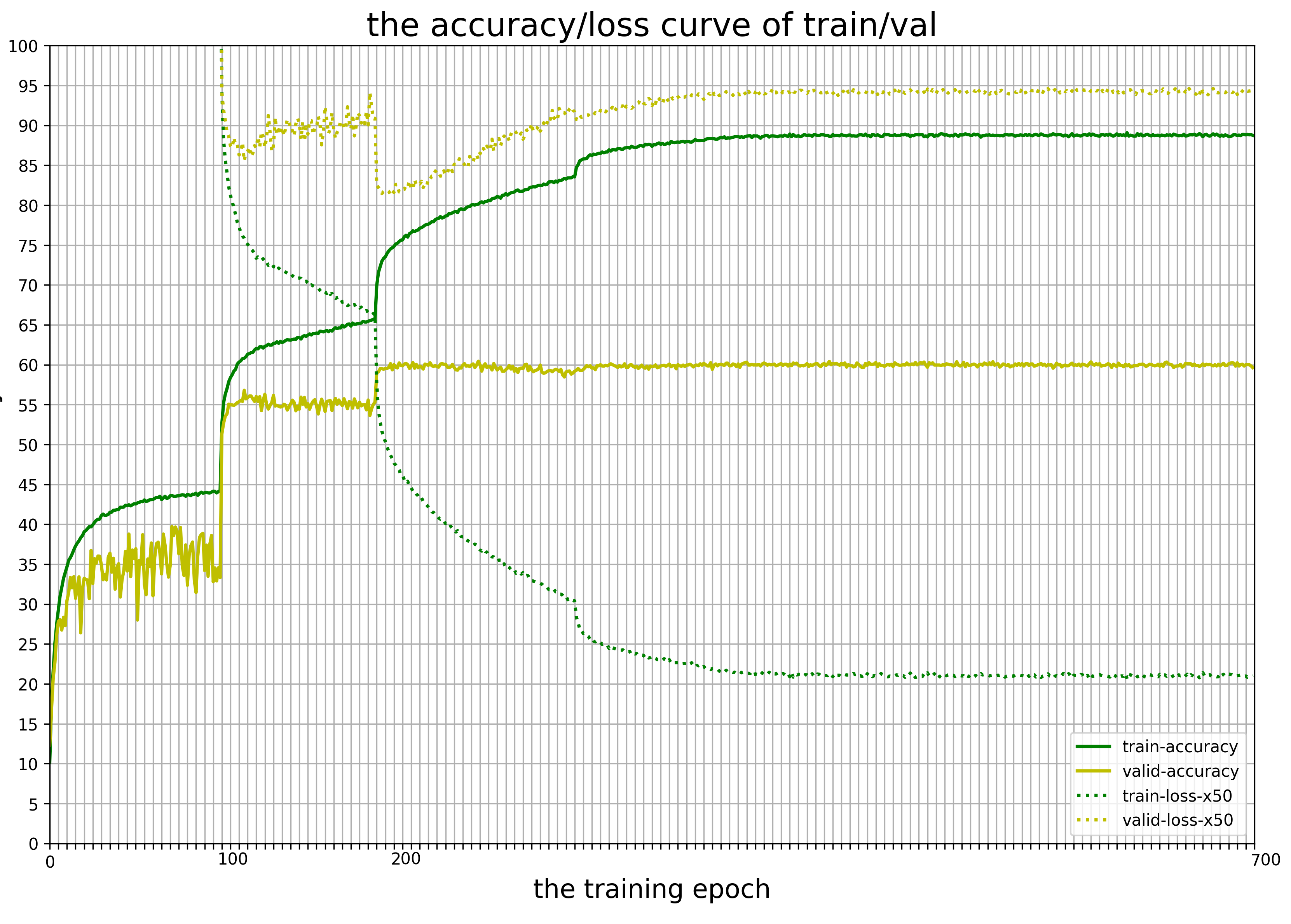}
    \caption{A plot of accuracy/loss per epoch for train/validation data on the original SimpleNet architecture, with 700 epochs. The loss is multiplied by 50. There are learning rate changes at epochs 100, 190, 306, 390, 440, 540, which correspond to the main jumps in train/valid accuracy. All experiments returned similar plots.}
    \label{fig:trainval}
\end{figure}

After analyzing Figure~\ref{fig:trainval}, we found that the algorithm overfits after epoch 200 (train accuracy increases and train loss decreases, but validation loss increases while validation accuracy plateaued). Then, instead of reporting validation accuracy at epoch 700, for each experiment we find the epoch with lowest validation loss, which occurs before epoch 200, and obtain the corresponding validation accuracy. 

For each architecture we report the average of the validation accuracy as described above. Note that the epoch in which the validation accuracy was measured may differ not only for different architectures but also within repetitions of the same experiment with different seeds.

Table~\ref{table:expimg100C1} contains the average (per architecture) validation accuracy. We found that replacing the max pooling at the first `p' layer with the $\pcauset[alt={x<y>w<z}]{2, 4, 1, 3}$ poset returned an average validation accuracy of $60.1667\%$ with std $0.58578$, compared to $59.85\%$ with std $0.4939$ in the original architecture.

\begin{table}[h]
\centering
\begin{tabular}{|c||c|c|c|}
\hline
    SimpleNet &  Training runs  &Avg. val acc at               & Std\\ 
    &&the lowest val loss&\\
    \hline\hline
      Original&6               & $59.8500\%$              & $0.4939$\\ 
      \hline
Replaced max pooling&6   &  $59.7233\%$              &$0.56447$ \\
in both `p' layers& &&\\
by $\pcauset[alt={x<y>w<z}]{2, 4, 1, 3}$ &&&\\\hline
Replaced max pooling&6   &  $\textbf{60.1667}\%$              & $0.58578$\\
in the first `p' layer& &&\\
by $\pcauset[alt={x<y>w<z}]{2, 4, 1, 3}$  &&&\\
\hline
Replaced max pooling&6   &  $59.7867\%$              &$0.3572$ \\
in the second `p' layer& &&\\
by $\pcauset[alt={x<y>w<z}]{2, 4, 1, 3}$ &&&\\
\hline
%\hline    
\end{tabular}
\caption{Experiment results on SimpleNet for ImageNet 100 classification task, replacing in different positions by the $\pcauset[alt={x<y>w<z}]{2, 4, 1, 3}$ poset.\label{table:expimg100C1}}
\end{table}

%Original: Val Acc 59.8500, std 0.49392307093311627 Prec@5 82.6033
%two_replacements_by_N: Val Acc 59.7233, std 0.5644702531282464 Prec@5 82.7067
%first_replacement_by_N: Val Acc 60.1667, std 0.5857872196170422 Prec@5 82.5033
%second_replacement_by_N: Val Acc 59.7867, std 0.35724874620727076 Prec@5 82.8467

The experiments in this section were conducted on the ICER Data Machine; the programs were run on NVIDIA A100 GPUs split into units with 10GB each.

\subsection{DenseNet}\label{sec:Dense}\label{Sec:exp2}

In this section, we report experiments on the DenseNet architecture of~\cite{DenseNET} with Fashion MNIST, CIFAR10 and CIFAR100 datasets.

We follow the Pytorch implementation of DenseNet \url{https://github.com/bamos/densenet.pytorch} by~\cite{DenseNET_p}. It has dense blocks separated by transitions layers. Transition layers include an average pooling layer. After the bottlenecks, there is another average pooling of size 8, followed by a dense section.

Our goal was to test the result of adding the $\pcauset[alt={x<y>w<z}]{2, 4, 1, 3}$ poset filter along the convolutional part of DenseNet. Unfortunately, the input of the dense layer has width and height 1 and adding a poset filter required major changes to the architecture. Then, we wondered if we could replace the several instances of average pooling layers of DenseNet by our poset filter.

We tested several combinations, but some did not converge. For example, the replacement of $F.avg\_pool2d(, 8)$ by $F.avg\_pool2d( \pcauset[alt={x<y>w<z}]{2, 4, 1, 3}(F.avg\_pool2d( ,2) ), 2)$ did not converge for Fashion MNIST, CIFAR10 or CIFAR100.

With respect to the Fashion MNIST dataset, when we run the original code we could only achieve an accuracy of $95.2593\%$ with std $0.141$, while the officially recorded value is of $95.4\%$. However, we note that we added seeds to make our results reproducible. Another difference from the GitHub code is that we added a padding of 2 to all sides of the Fashion MNIST dataset images, since the GitHub code is designed for the CIFAR dataset, which has different dimensions. Our best experiment (replacing the average pooling by the poset filter on the second transition layer) returned an average of $95.2864\%$ with $0.134$ std. While we obtained a higher test accuracy than the original architecture, if instead of the $\pcauset[alt={x<y>w<z}]{2, 4, 1, 3}$ filter we use a max pooling (replacing the average pooling by the max pooling filter on the second transition layer), then we obtain an even higher accuracy of $95.292\%$ with $0.122$ std. Adding the $\pcauset[alt={x<y>w<z}]{2, 4, 1, 3}$ filter does not affect the accuracy in this case, but for certain combinations of data and architectures other filters may be a better fit. We also tested removing the ReLU, reaching the same conclusion.

We conducted the same experiments with CIFAR10, replacing the average pooling layers with the $\pcauset[alt={x<y>w<z}]{2, 4, 1, 3}$ filter, but our best result ($95.081\%$ test accuracy with a standard deviation of $0.162$), obtained by replacing the average pooling in the first transition by the $\pcauset[alt={x<y>w<z}]{2, 4, 1, 3}$ filter, is lower than the accuracy for the original architecture ($95.128\%$ with $0.119$ std). In this case, replacing the average pooling in the first transition by a max pooling returned an average accuracy of $95.216\%$ with $0.12$ std.

Similarly, with CIFAR100 our best result was obtained by replacing the average pooling in the first transition by the $\pcauset[alt={x<y>w<z}]{2, 4, 1, 3}$ filter, with an average accuracy of $76.93\%$ and $0.258$ std, while the original architecture returned $76.989\%$ with $0.288$ std. Instead, when replacing the average pooling in the first transition by a max pooling, we obtained an accuracy of $77.13\% $ with a standard deviation of $0.294$.

All experiments were performed on NVIDIA A100-SXM4-80GB.

\subsection{Filter functions}
\label{A:1}
Consider the transformation, which we will call \pcauset[alt={x, y, w, z}]{4, 3, 2, 1} filter (disjoint union of points), that sends the $(2\times 2)$ square matrix ${\begin{bmatrix}
   a_{0,0} & a_{0,1} \\
   a_{1,0} & a_{1,1} \\
  \end{bmatrix} }$
to the maximum of all possible sums out of the four inputs, as seen in Equation~\eqref{eqn:cube2}. 

% \begin{equation*}
% \max\left\{
% \begin{aligned}
%    & \ \ \ \ 0, \\
%    & \ \ \ \max_{i,j}\{a_{i,j}\}, \\
%    & \max_{\substack{i,j,k,l\\(i,j)\neq (k,l)} }\{a_{i,j}+a_{k,l}\}, \\
%    & \max_{\substack{i,j,k,l,m,n\\
% (i,j)\neq (k,l)\\(i,j)\neq (m,n)\\(k,l)\neq (m,n)
% }}\{a_{i,j}+a_{k,l}+a_{m,n}\}, \\
% & \ \ \ \ a_{0,0}+a_{1,0}+a_{0,1}+a_{1,1}\\
% \end{aligned} 
% \right\}.
% \end{equation*}

This transformation returns the partial sum that has the largest value. Following~\cite{CNNF}, we can think of this operation a geometric transformation on an input image, assuming that the pixels have values between $(-1,1)$ and that we normalize the values before displaying the image. 

There are two drawbacks to this transformation: first, the image needs to be normalized after the transformation. Secondly, it contains many operations. 

In contrast, we also consider the following transformation, which we will call \pcauset[alt={4-\hbox{chain} }]{1, 2, 3, 4} filter:
$$\max\{0,a_{0,0},
a_{0,0}+a_{0,1},
a_{0,0}+a_{0,1}+a_{1,1},a_{0,0}+a_{1,0}+a_{0,1}+a_{1,1}\}.$$  

We conducted experiments on an image of $4,868\times 3,245$ pixels, uploaded to Wikimedia by Diego Delso, delso.photo, License CC BY-SA. In Figure~\ref{fig:motivation}, we plot the resultant image after applying the \pcauset[alt={x, y, w, z}]{4, 3, 2, 1} filter and the \pcauset[alt={4-\hbox{chain} }]{1, 2, 3, 4} filter. We also plot the effect of max pooling and average pooling to compare the effects.

    \begin{figure}[h]
        \centering
        \begin{subfigure}[b]{0.45\textwidth}
            \centering
            \includegraphics[width=\textwidth]{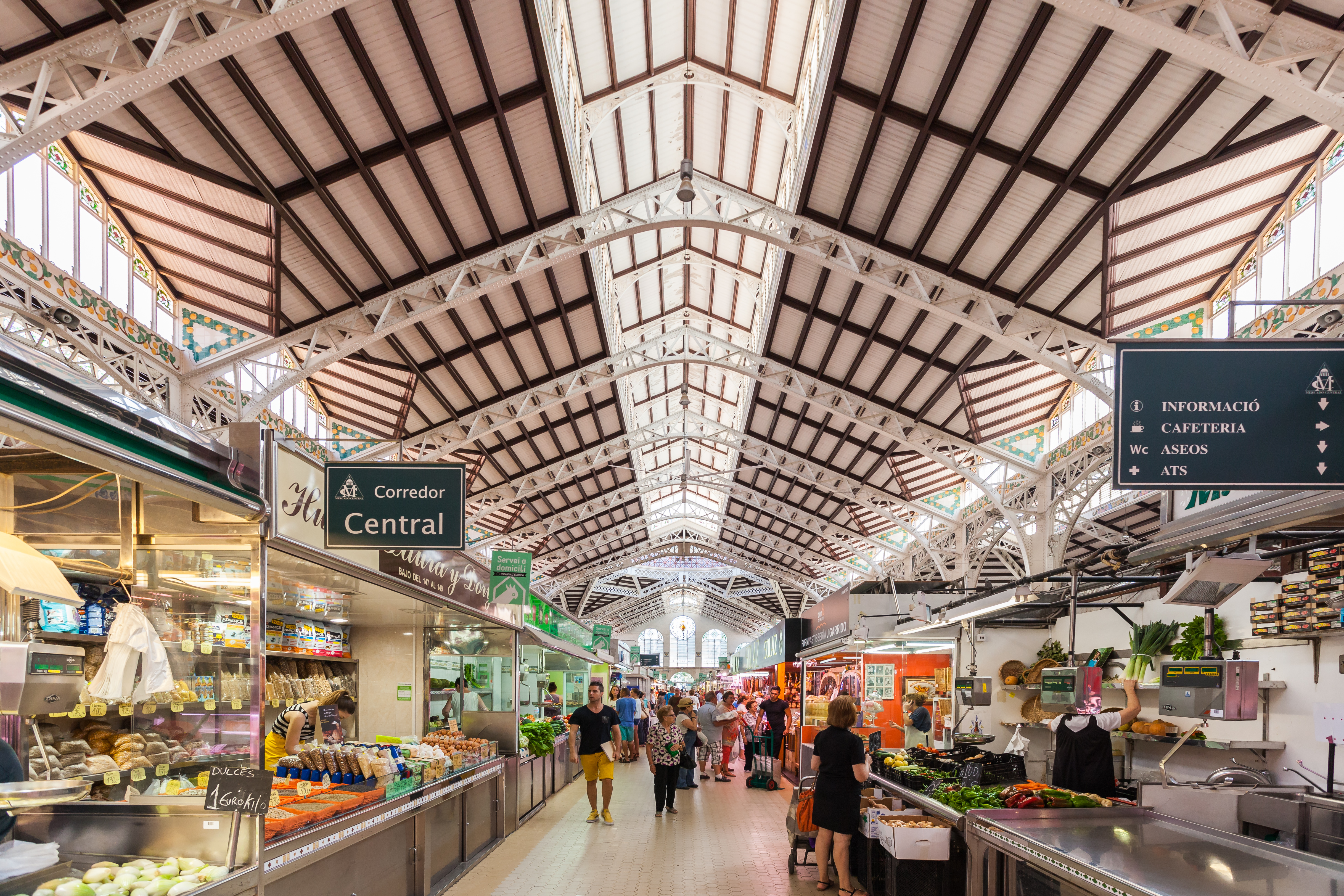}
            \caption[]%
            {{\small Original\\ image}}    
            \label{fig:mean and std of net14}
        \end{subfigure}
          \vskip\baselineskip
        \begin{subfigure}[b]{0.3\textwidth}  
            \centering 
            \includegraphics[width=\textwidth]{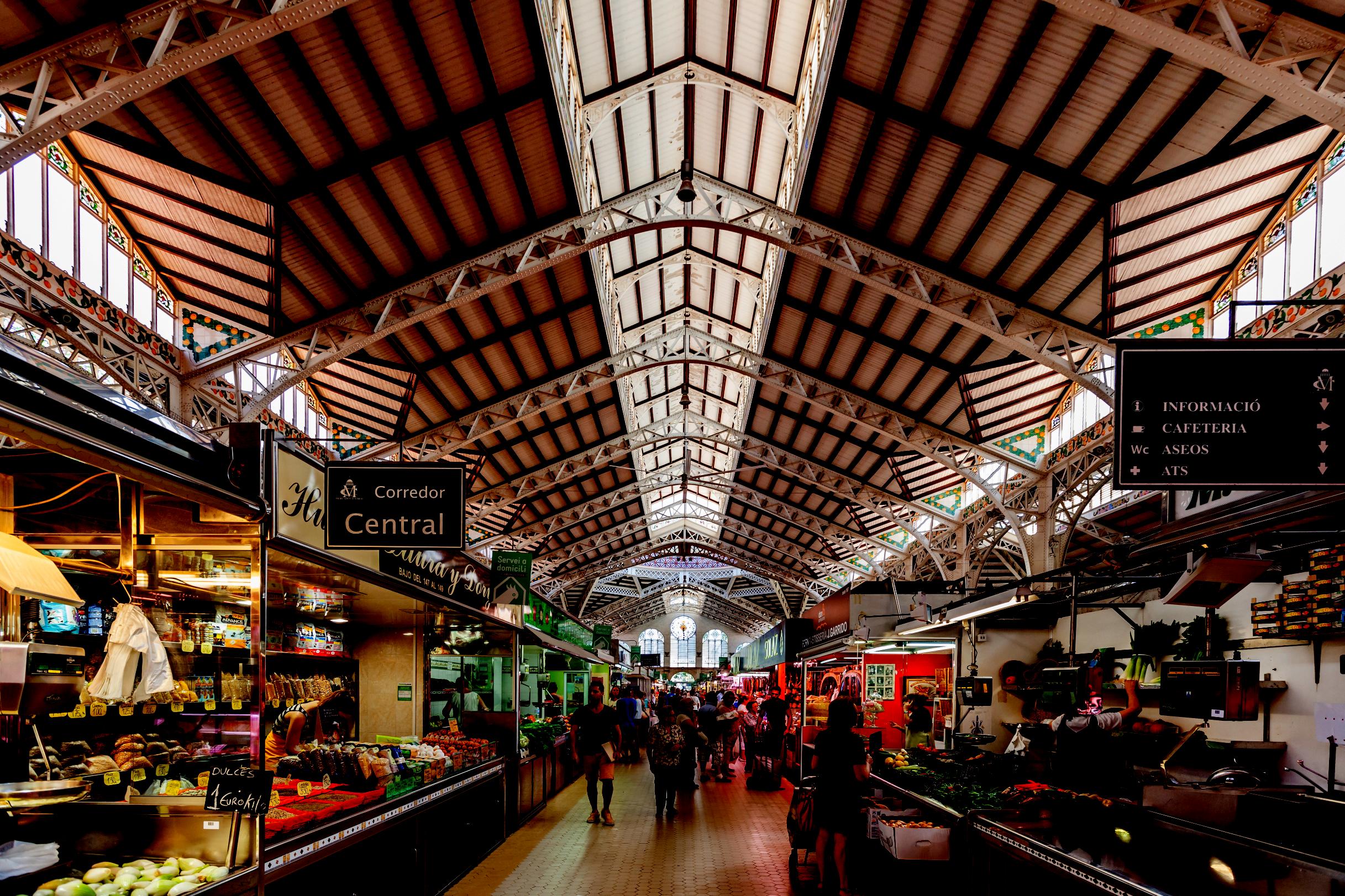}
            \caption[]%
            {{\small Effect of \pcauset[alt={4-\hbox{chain} }]{1, 2, 3, 4}\\ filter}}    
            \label{fig:mean and std of net24}
        \end{subfigure}
        \hfill
        \begin{subfigure}[b]{0.3\textwidth}   
            \centering 
            \includegraphics[width=\textwidth]{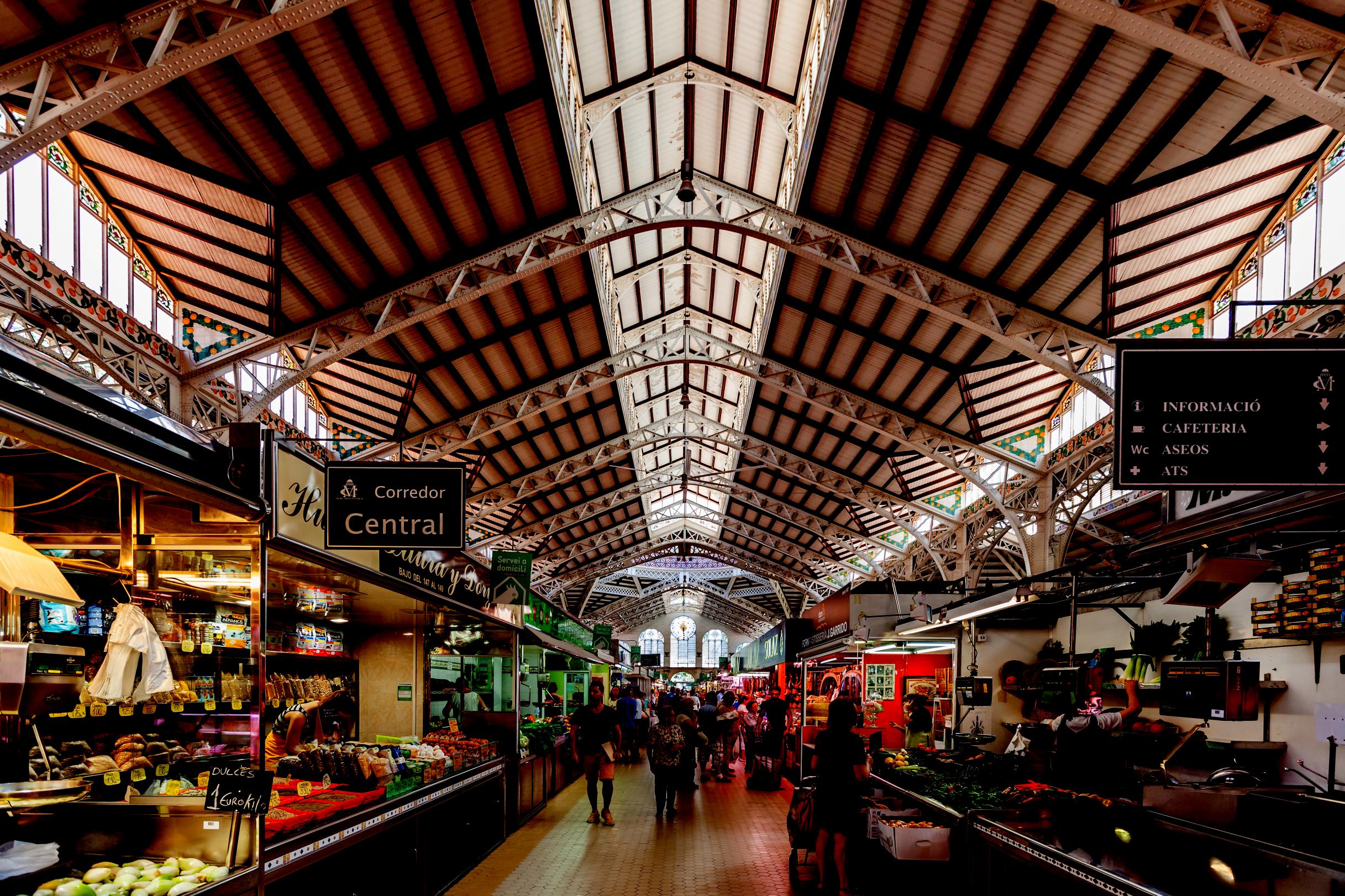}
            \caption[]%
            {{\small Effect of \pcauset[alt={x, y, w, z}]{4, 3, 2, 1}\\ filter}}    
            \label{fig:mean and std of net34}
        \end{subfigure}          
        \vskip\baselineskip
        \begin{subfigure}[b]{0.3\textwidth}  
            \centering 
            \includegraphics[width=\textwidth]{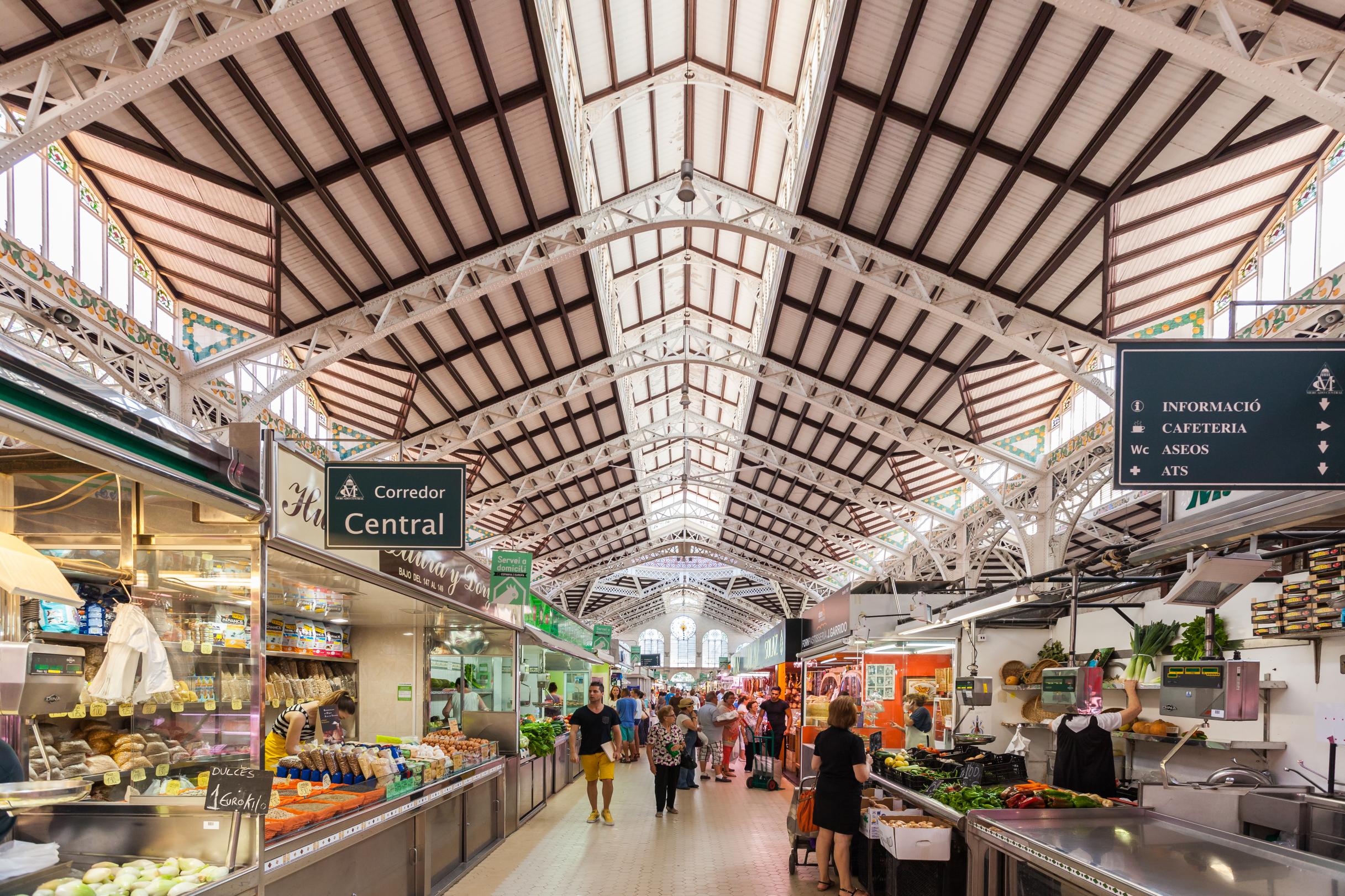}
            \caption[]%
            {{\small Effect of \\ average filter}}    
            \label{fig:mean and std of net44}
        \end{subfigure}
        \hfill
        \begin{subfigure}[b]{0.3\textwidth}   
            \centering 
            \includegraphics[width=\textwidth]{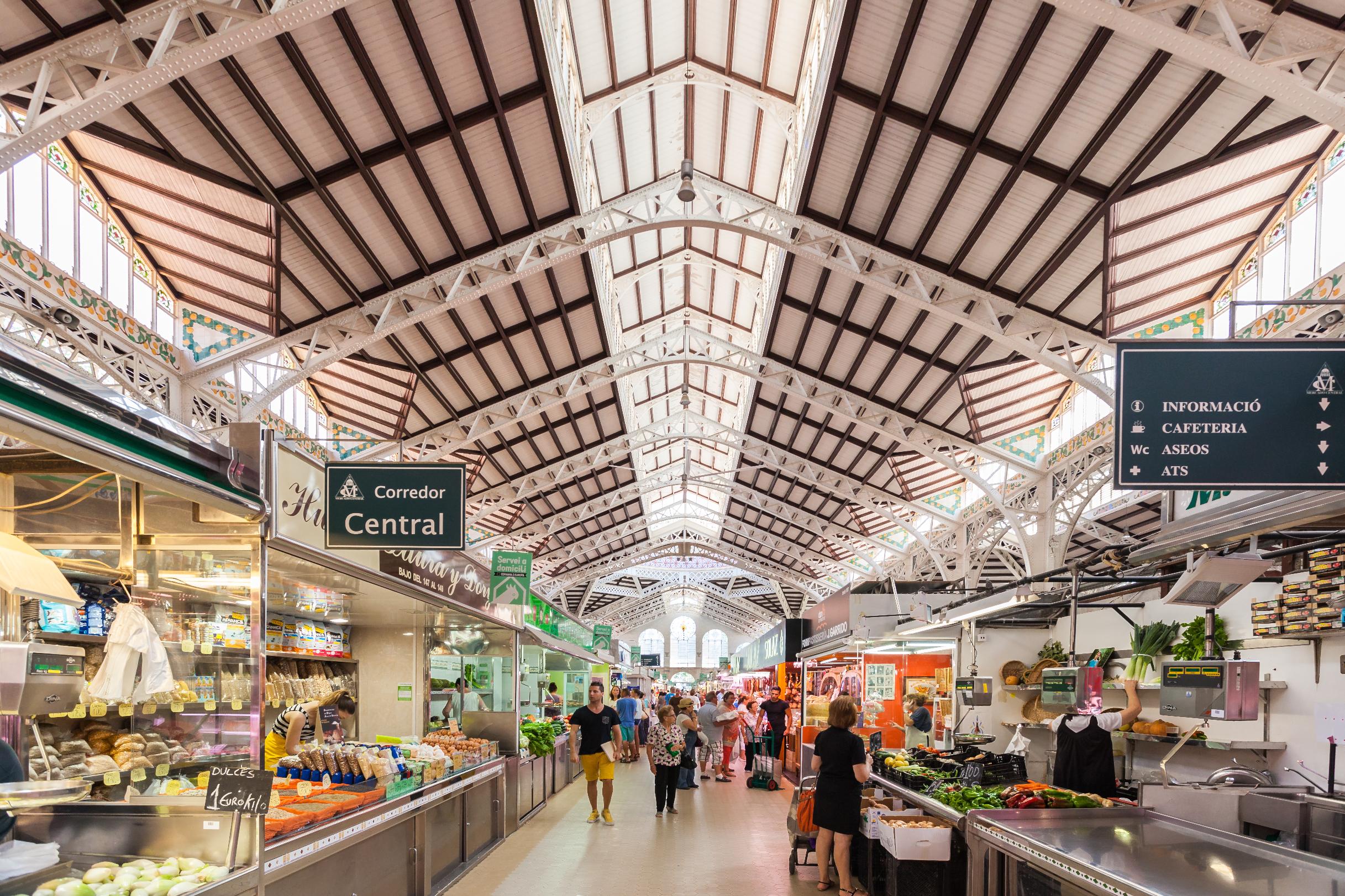}
            \caption[]%
            {{\small Effect of \\ max filter}}    
            \label{fig:mean and std of net54}
        \end{subfigure}
        \caption[] 
         {\small Images b), c), d), and e) have half the dimensions of image a). }
        \label{fig:motivation}
    \end{figure}

To find if there is any difference between the \pcauset[alt={x, y, w, z}]{4, 3, 2, 1} filter and the \pcauset[alt={4-\hbox{chain} }]{1, 2, 3, 4} filter, we perform the following experiment: taking the original image, apply one fixed filter three times, effectively resizing the image to 1/8 of the original dimension. Then, use nearest neighbors to resize the image to the original shape. In this way, we obtain two images: image $A_{\pcauset[alt={x, y, w, z}]{4, 3, 2, 1}}$, obtained from applying the \pcauset[alt={x, y, w, z}]{4, 3, 2, 1} filter three times and then upsizing, and image $B_{\pcauset[alt={4-\hbox{chain} }]{1, 2, 3, 4}}$, obtained from applying the \pcauset[alt={4-\hbox{chain} }]{1, 2, 3, 4} filter three times and then upsizing. Then, we compare the SSIM and PSNR between image $A_{\pcauset[alt={x, y, w, z}]{4, 3, 2, 1}}$ and the original image, and between image $B_{\pcauset[alt={4-\hbox{chain} }]{1, 2, 3, 4}}$ and the original image. As expected from the visual evidence, other methods may be more convenient for resizing, but we only aim to obtain an understanding of the effect of the filters.

We present the statistics obtained by applying the nearest-neighbor method three times, followed by upsizing using the same approach. In addition, we report the corresponding SSIM and PSNR values.
\begin{center}
\begin{tabular}{|c ||c |c |} 
 \hline
 & SSIM & PSNR  \\ [0.5ex] 
 \hline\hline
 Figure $A_{\pcauset[alt={x, y, w, z}]{4, 3, 2, 1}}$ and original & .1748 & 5.44 \\ 
 \hline
 Figure $B_{\pcauset[alt={4-\hbox{chain} }]{1, 2, 3, 4}}$ and original & .1756 & 5.56  \\
 \hline
 Nearest N and original &.248 & 7.41 \\
 \hline
\end{tabular}
\end{center}

% {\noindent \em Remainder omitted in this sample. See http://www.jmlr.org/papers/ for full paper.}

\clearpage
%\bibliography{sample}

\bibliography{sample.bib}{}
\bibliographystyle{siam}

%\printbibliography

\end{document}